\documentclass[journal,twoside,web]{ieeecolor}

\usepackage{generic}
\usepackage[noadjust]{cite}
\usepackage{amsmath,amssymb,amsfonts}
\usepackage{algorithmic}
\usepackage{graphicx}
\usepackage{textcomp}
\usepackage{dsfont}

\usepackage{graphicx}
\usepackage{subfigure}
\usepackage{epsfig} %
\usepackage{doi}
\usepackage{comment}

\usepackage{tabularx}
\usepackage{layouts}
\usepackage{booktabs}
\usepackage{multirow}

\usepackage{amsthm}  %
\usepackage{lcsys}
\usepackage{etoolbox}
\usepackage{cancel}

\usepackage[capitalise]{cleveref}

\crefname{assumption}{Assumption}{Assumptions}

\newtheorem{theorem}{Theorem}
\newtheorem{lemma}{Lemma}

\newtheorem{corollary}{Corollary}
\newtheorem{definition}{Definition}

\newtheorem{assumption}{Assumption}

\begin{document}

\def\BibTeX{{\rm B\kern-.05em{\sc i\kern-.025em b}\kern-.08em
    T\kern-.1667em\lower.7ex\hbox{E}\kern-.125emX}}
\markboth{\journalname, VOL. XX, NO. XX, XXXX 2017}
{Author \MakeLowercase{\textit{et al.}}: Preparation of Papers for IEEE Control Systems Letters (August 2022)}

\newcommand{\kf}{k}
\newcommand{\Hk}{\mathcal{H}_k}
\newcommand{\Hkf}{\Hk}

\newcommand{\ks}{k^{\sigma}}
\newcommand{\kw}{k^w}
\newcommand{\ksum}{\kf + \ks}
\newcommand{\kd}{k^{\delta}}

\newcommand{\Hks}{\mathcal{H}_{\ks}}

\newcommand{\fmu}[1][\sigma]{\mu_{#1}}
\newcommand{\wmu}[1][\sigma]{w^\mu_{#1}}
\newcommand{\gmu}[1][\sigma]{g^\mu_{#1}}
\newcommand{\thetamu}[1][\sigma]{\theta^\mu_{#1}}
\newcommand{\fopt}{f^\star}
\newcommand{\wopt}{w^\star}
\newcommand{\gopt}{g^\star}
\newcommand{\thetaopt}{\theta^\star}
\newcommand{\ftr}{f^{\mathrm{tr}}}
\newcommand{\wtr}{w^{\mathrm{tr}}}
\newcommand{\thetatr}{\theta^{\mathrm{tr}}}
\newcommand{\ftrint}{f^{\mathrm{tr,int}}}
\newcommand{\wtrint}{w^{\mathrm{tr,int}}}

\newcommand{\Rgeq}{\mathbb{R}_{\geq 0}}
\newcommand{\Rge}{\mathbb{R}_{> 0}}

\newcommand{\covar}[1][\lambda]{\Sigma_{#1}}
\newcommand{\Mu}[1][\lambda]{m_{#1}}
\newcommand{\Sig}[1][\lambda]{S_{#1}}
\newcommand{\Gram}[1][\lambda]{M_{#1}}
\newcommand{\gram}[1][\lambda]{\hat{K}_{#1}}
\newcommand{\Sump}[1][\lambda]{Q_{#1}}

\newcommand{\Gam}[1][\lambda]{\Gamma_{#1}}
\newcommand{\gam}[1][\lambda]{\Gamma_{#1}}
\newcommand{\Beta}[1][\lambda]{\beta_{#1}}
\newcommand{\bet}[1][\lambda]{\beta_{#1}}

\newif\ifnotes
\notesfalse

\newtoggle{arxiv}
\toggletrue{arxiv}

\newif\ifarxiv
\iftoggle{arxiv}{\arxivtrue}{\arxivfalse}

\title{Optimal uncertainty bounds for multivariate kernel regression under bounded noise: \\ A Gaussian process-based dual function}

\author{
  Amon Lahr$^{a}$,
  Anna Scampicchio$^{b,*}$,
  Johannes Köhler$^{c,*}$,
  Melanie N. Zeilinger$^{a}$
  \thanks{$^{a}$ Institute for Dynamical Systems and Control, ETH Zurich, Zurich, Switzerland; \texttt{\{amlahr,mzeilinger\}@ethz.ch}}
  \thanks{$^{b}$ Department of Electrical Engineering, Chalmers University of Technology, Göteborg, Sweden; \texttt{anna.scampicchio@chalmers.se}}
  \thanks{$^{c}$ Department of Mechanical Engineering, Imperial College London, London, UK; \texttt{j.kohler@imperial.ac.uk}}
  \thanks{$^{*}$ Both co-authors contributed equally.}
}

\maketitle
\thispagestyle{empty}

\begin{abstract}
    Non-conservative uncertainty bounds are essential for making reliable predictions about latent functions from noisy data---and thus, a key enabler for safe learning-based control. In this domain, kernel methods such as Gaussian process regression are established techniques, thanks to their inherent uncertainty quantification mechanism. Still, existing bounds either pose strong assumptions on the underlying noise distribution, are conservative, do not directly apply in the multi-output case, or are difficult to integrate into downstream tasks. 
    This paper addresses these limitations by presenting a tight, deterministic bound for multi-output functions in Reproducing Kernel Hilbert Spaces~(RKHSs) subject to bounded noise.
    It is obtained through an
    unconstrained, duality-based formulation,
    which
    shares the same structure as classic Gaussian process confidence bounds,
    and can thus be straightforwardly integrated
    into downstream optimization pipelines. 
    We show that the proposed bound
    generalizes existing results 
    and illustrate 
    its application
    using an example inspired by quadrotor dynamics learning.
\end{abstract}


\section{Introduction}

\IEEEPARstart{K}{ernel}
methods have 
attracted a lot of attention 
in control
due to their accurate estimation performance, flexibility and low tuning effort~\cite{scholkopf_learning_2002}. In particular, estimates from Gaussian process~(GP) regression~\cite{rasmussen_gaussian_2006}, being complemented by rigorous uncertainty bounds, are very promising in view of learning-based control with safety constraints~\cite{scampicchio_gaussian_2025,wabersich_data-driven_2023,martin_guarantees_2023}. However, standard confidence bounds for Gaussian process regression~\cite{srinivas_information-theoretic_2012,abbasi-yadkori_online_2013,chowdhury_kernelized_2017,fiedler_safety_2024-1,molodchyk_towards_2025-1} assume independent noise realizations and cannot handle correlated noise sequences. 
Deterministic bounds~\cite{maddalena_deterministic_2021,reed_error_2025,hashimoto_learning-based_2022-1,yang_kernel-based_2024,scharnhorst_robust_2023,lahr_optimal_2025} alleviate this independence assumption,
providing robust uncertainty envelopes for bounded noise realizations without asserting a particular noise distribution.
Yet, 
existing 
results
are either conservative~\cite{maddalena_deterministic_2021,hashimoto_learning-based_2022-1,yang_kernel-based_2024,reed_error_2025}, require constrained or non-smooth optimization~\cite{scharnhorst_robust_2023}, or are restricted to simple 
hyper-cube or ellipsoidal
uncertainty sets~\cite{maddalena_deterministic_2021,reed_error_2025,hashimoto_learning-based_2022-1,yang_kernel-based_2024,scharnhorst_robust_2023,lahr_optimal_2025}.

In this work, we derive
a tight, deterministic uncertainty bound 
for noise bounded by an intersection of ellipsoids,
encompassing
point-wise bounded and energy-bounded noise realizations.
The result is
given 
in terms of the
unconstrained optimizer of
an invex, Gaussian process-based dual function (\cref{sec:main_result}),
and is directly applicable in the multi-output case with partial measurements.
We provide a thorough comparison with 
other 
existing deterministic bounds,
showing that our bound
encompasses those in~\cite{hashimoto_learning-based_2022-1,yang_kernel-based_2024,scharnhorst_robust_2023,lahr_optimal_2025} as special cases (\cref{sec:discussion}).
Finally, we 
numerically 
compare the bounds' 
conservatism and solve time 
using an application example inspired by 
multivariate 
quadrotor dynamics learning
subject to direction-dependent wind disturbances (\cref{sec:application}).
We provide a further generalization of the proposed uncertainty bounds, including a convex formulation, in 
\iftoggle{arxiv}{%
\cref{sec:app_generalized_bounds}.
}{%
\cite[Appendix~E]{lahr_optimal_2026}.
}

\subsubsection*{Notation}

The space of (strictly) positive real numbers 
is denoted by $\Rgeq$ ($\Rge$).
We denote by $e_i \in \mathbb{R}^{N}$ the \mbox{$i$-th} unit vector in $\mathbb{R}^{N}$, 
and by $\mathds{1}_N \in \mathbb{R}^{N}$ a vector of ones.
The weighted Euclidean norm 
is denoted by $\| v \|_M \doteq \sqrt{v^\top M v}$.
The 
concatenation of
matrices $M_i \in \mathbb{R}^{n \times m}$ 
is written
as \mbox{$[M_i]_{i=1}^N \doteq [M_1^\top, \ldots, M_N^\top]^\top \in \mathbb{R}^{N n \times m}$}. 
Likewise,
for two 
ordered 
index sets
\mbox{$\mathbb{I},\mathbb{J}\subseteq\mathbb{N}$} with cardinalities $|\mathbb{I}| = N$, $|\mathbb{J}| = M$, 
particularly for e.g. $\mathbb{I}_1^N \doteq 1 : N \doteq \{1,\dots,N\}$,
the matrix \mbox{$K^{f}_{\mathbb{I},\mathbb{J}} \doteq [k(x_i,x_j)]_{i \in \mathbb{I}, j \in \mathbb{J}} \in \mathbb{R}^{N n_f \times M n_f}$} 
collects
the evaluations of the (matrix-valued) kernel function \mbox{$k: \mathbb{R}^{n_x} \times \mathbb{R}^{n_x} \rightarrow \mathbb{R}^{n_f \times n_f}$} at pairs of input locations $x_i$, $x_j$.

\section{Problem setting}\label{sec:problem_setting}

We 
consider the problem of
computing uncertainty bounds
for the value of an
unknown
multivariate
function 
$\ftr: \mathbb{R}^{n_x} \rightarrow \mathbb{R}^{n_f}$ at an arbitrary input location~$x_{N+1} \in \mathbb{R}^{n_x}$.
Information about the latent function is given
in terms of a data set
of $N$
noisy, partial measurements $y = [y_i]_{i=1}^N \in \mathbb{R}^{N}$,
\begin{align}
  \label{eq:data}
  y_i = c_i^\top \ftr(x_i) + w_i, \> \forall i \in \mathbb{I}_1^N
\end{align}
corrupted by
additive noise~$w_i$.
We assume that the disturbances 
$w = [w_i]_{i=1}^{N} \in \mathbb{R}^{N}$
are jointly bounded by a collection of ellipsoidal uncertainty sets, 
defined as follows.
\begin{assumption}
  \label{ass:bounded_noise}
  The noise realizations are bounded by
  \begin{align}
    \label{eq:noise_constraints}
    w^\top P^w_j w  < \Gamma_{w,j}^2 \> \forall 
    j \in \mathbb{I}_1^{n_{\mathrm{con}}},
  \end{align}
  with known
  constants 
  $\Gamma_{w,j} > 0$ and positive-semidefinite matrices 
  $P_j^w \in \mathbb{R}^{N \times N}$.
  Additionally, 
  the matrix 
  $P^w \doteq \sum_{j=1}^{n_\mathrm{con}} P^w_j$ is positive definite.
\end{assumption}
\noindent 
The constraints~\eqref{eq:noise_constraints}
can model
correlated disturbances 
across input locations and output components.
Special cases 
include:
\begin{enumerate}
  \item point-wise bounded noise:
        For $n_{\mathrm{con}} = N$ and 
        $P^w_j = e_j e_j^\top$, 
        \cref{eq:noise_constraints}
        reads as 
        $w_j^2 < \Gamma_{w,j}^2$, 
        $j \in \mathbb{I}_1^N$.
  \item energy-bounded noise:
        For $n_{\mathrm{con}} = 1$, 
        \cref{eq:noise_constraints}
        reads as 
        $w^\top P^w_1 w < \Gamma_{w,1}^2$, 
        with $P^w_1$ positive definite. The special case $P^w_1 = I_{N}$ recovers the setting of energy-bounded noise realizations, i.e., the constraint 
        $\sum_{i=1}^{N} w_i^2 < \Gamma_{w,1}^2$.
\end{enumerate}

To infer uncertainty bounds about the latent function~$f$,
we make use of the following standard regularity assumption~\cite{srinivas_information-theoretic_2012,abbasi-yadkori_online_2013,chowdhury_kernelized_2017,fiedler_safety_2024-1,molodchyk_towards_2025-1,maddalena_deterministic_2021,reed_error_2025,hashimoto_learning-based_2022-1,yang_kernel-based_2024,scharnhorst_robust_2023,lahr_optimal_2025}, 
cf.~\cite[Remark~1]{scharnhorst_robust_2023} for a discussion.
\begin{assumption}
  \label{ass:rkhs_norm_f}
  Let 
  $\ftr \in \mathcal{H}_k$ be an element of the Reproducing Kernel Hilbert Space (RKHS) $\Hk$, defined by a given positive-semidefinite, matrix-valued kernel function \mbox{$\kf: \mathbb{R}^{n_x} \times \mathbb{R}^{n_x} \rightarrow \mathbb{R}^{n_f} \times \mathbb{R}^{n_f}$}.
  Let the norm of $\ftr$
  be bounded, i.e.,
  $\| \ftr \|_{\Hk} < \Gamma_{f}$,
  for a known constant $\Gamma_f > 0$.
\end{assumption}
This general set-up 
encompasses latent-function estimation using 
multivariate finite-dimensional features (through positive semi-definiteness of~$\kf$), 
independent scalar outputs (through a diagonal structure of $\kf$),
or any combination thereof.

\section{Main result}
\label{sec:main_result}

We formulate the desired uncertainty bound for the vector-valued unknown function 
in terms of
its test-point-dependent, worst-case realization for an arbitrary direction $h \in \mathbb{R}^{n_f} \setminus \{ 0 \}$.
To this end, we transcribe \cref{ass:bounded_noise,ass:rkhs_norm_f} as well as \cref{eq:data}
into the following optimization problem: 
\begin{subequations}
  \label{eq:inf_opt_infdim}
  \begin{align}
    \overline{f}_h(x_{N+1}) = \sup_{\substack{f \in \Hkf                                        \\ w \in \mathbb{R}^{N}}} \quad & h^\top f(x_{N+1}) \\
    \mathrm{s.t.} \quad 
    &
    c_i^\top f(x_i) + w_i = y_i, \> i \in \mathbb{I}_1^N, \label{eq:inf_opt_infdim_data} \\
    & w^\top P^w_j w \leq \Gamma_{w,j}^2, \label{eq:inf_opt_infdim_noise} 
    \> j \in \mathbb{I}_1^{n_{\mathrm{con}}}, \\
    & \| f \|_{\Hkf}^2 \leq \Gamma_{f}^2. \label{eq:inf_opt_infdim_rkhs}
  \end{align}
\end{subequations}
By definition, an optimal solution of Problem~\eqref{eq:inf_opt_infdim} 
determines the optimal worst-case bound for the value of 
\mbox{$\ftr(x_{N+1})$}
in the direction~$h$, 
subject to the available information.
Consequently,
the tightest containment interval
along 
this axis 
is 
$-\overline{f}_{-h}(x_{N+1}) \leq h^\top \ftr(x_{N+1}) \leq \overline{f}_h(x_{N+1})$.

Our main result derives an optimal solution to Problem~\eqref{eq:inf_opt_infdim} 
using
familiar terms from GP regression,
using a particular 
measurement noise covariance.
To this end,
we
denote 
the multi-output 
GP
posterior mean and covariance 
by
\begin{align}
  \fmu(x_{N+1}) &\doteq K^f_{N+1,1:N} C \hat{K}_\sigma^{-1} y, 
  \label{eq:gp_definitions}
    \\
    \Sigma_\sigma(x_{N+1}) &\doteq 
    K^f_{N+1,N+1} - K^f_{N+1,1:N} C \hat{K}_\sigma^{-1} C^\top K^f_{1:N,N+1}, \notag
\end{align}
respectively.
The matrix
\mbox{$\hat{K}_\sigma \doteq C^\top K^f_{1:N,1:N} C + K^w_\sigma$}
denotes the Gram matrix,
where the matrix~$K^f_{1:N,1:N}$ of kernel evaluations on the training inputs is projected 
using the measurement matrix~$C^\top \doteq \mathrm{blkdiag}\left( 
    c_1^\top, \ldots, c_N^\top
    \right) \in \mathbb{R}^{N \times n_f N}$, 
defined according to the linear measurement model~\eqref{eq:data}.
While~$K^w_\sigma$ commonly denotes
the 
covariance 
of (possibly correlated) Gaussian measurement noise,
here
we define
$
K^w_\sigma \doteq \left( P^w_\sigma \right)^{-1} 
$,
with
$P^w_\sigma
\doteq  
  \sum_{j=1}^{n_{\mathrm{con}}} \sigma_j^{-2} P^w_j 
$
based on the ellipsoidal noise bounds~\eqref{eq:noise_constraints}
and a \emph{free}
vector of noise parameters
\mbox{$\sigma = [\sigma_j]_{j=1}^{n_{\mathrm{con}}} \in \Rge^{n_\mathrm{con}}$}.
For special case~1 of \cref{ass:bounded_noise}, 
note that 
\mbox{$K^w_\sigma = \mathrm{diag}(\sigma_1^2, \ldots, \sigma_N^2)$} 
corresponds to 
assuming independent, heteroscedastic noise;
for special case~2,
i.e.,
\mbox{$K^w_\sigma = \sigma_1^2 \left( P^w_1 \right)^{-1}$},
the scalar parameter~$\sigma_1$ can be interpreted as the output scale for an unknown noise-generating process~\cite[Sec.~3.1]{lahr_optimal_2025}.

We are now ready to state our main result.
\begin{theorem}
  \label{thm:optimal_bound}
  Let \cref{ass:bounded_noise,ass:rkhs_norm_f} hold and define
  \begin{align}
    \label{eq:optimal_bound_expression}
    \overline{f}_h^\sigma(x_{N+1}) &\doteq h^\top \fmu(x_{N+1}) + \beta_\sigma \sqrt{h^\top \Sigma_\sigma(x_{N+1}) h},
  \end{align}
  with
  \begin{align}
    \beta_\sigma &\doteq \sqrt{
        \Gamma_f^2 + \sum_{j=1}^{n_{\mathrm{con}}} \frac{\Gamma_{w,j}^2}{\sigma_j^2} - \| y \|_{\hat{K}_\sigma^{-1}}^2.
      }
  \end{align}
  Then, for any $h \in \mathbb{R}^{n_f} \setminus \{ 0 \}$ and any $x_{N+1} \in \mathbb{R}^{n_x}$,
  \begin{enumerate}
    \item[i)] $h^\top \ftr(x_{N+1}) \leq \overline{f}_h^\sigma(x_{N+1})$
      for all 
      $\sigma \in \Rge^{n_{\mathrm{con}}}$; 
    \item[ii)]
      the optimal 
      uncertainty bound~\eqref{eq:inf_opt_infdim} 
      is given by
      \begin{align}
        \label{eq:inf_opt_result}
        \overline{f}_h(x_{N+1}) = \inf_{
          \sigma \in \Rge^{n_{\mathrm{con}}}
        } \quad & 
        \overline{f}_h^\sigma(x_{N+1});
      \end{align}
    \item[iii)] 
      $\overline{f}_h^\sigma(x_{N+1})$ is invex: $\sigma^\star \in \Rge^{n_{\mathrm{con}}}$ is a global minimizer of Problem~\eqref{eq:inf_opt_result} if and only if $\left. \frac{\partial}{\partial \sigma}  \overline{f}_h^{\sigma} (x_{N+1}) \right|_{\sigma=\sigma^\star} = 0$.
  \end{enumerate}
\end{theorem}

The proof of \cref{thm:optimal_bound}, reported in \cref{sec:app_optimal_proof}, mainly relies on strong duality of a finite-dimensional representation of Problem~\eqref{eq:inf_opt_infdim}. 
In particular,
the function
$\overline{f}_h^\sigma(x_{N+1})$, 
which corresponds to \mbox{$\beta_\sigma$-confidence} bounds of the multivariate GP posterior in the direction~$h$,
is shown to be
a reparametrization of the
\emph{dual function} of Problem~\eqref{eq:inf_opt_infdim},
establishing
statements i) and ii) by weak and strong duality, respectively.
For the reparametrization,
the optimal dual variable~$\lambda_0$ for the complexity constraint~\eqref{eq:inf_opt_infdim_rkhs} is determined in closed form and 
the noise parameters 
$[\sigma_j^2]_{j=1}^{n_{\mathrm{con}}} \doteq [\frac{\lambda_0}{\lambda_j}]_{j=1}^{n_{\mathrm{con}}}$ 
are 
inversely proportional to
the 
dual variables 
$[\lambda_j]_{j=1}^{n_{\mathrm{con}}}$ 
for each respective constraint~\eqref{eq:inf_opt_infdim_noise}.
The scaling factor~$\beta_\sigma$ can be interpreted as the total RKHS norm in a data-generating RKHS, 
subtracted by the RKHS norm of the corresponding minimum-norm interpolant, cf.~\cite[Eq.~(6)]{lahr_optimal_2025}.

Statement i) of \cref{thm:optimal_bound}
establishes 
a valid, 
potentially conservative uncertainty bound for all noise parameters~$\sigma \in \Rge^{n_{\mathrm{con}}}$.
Noting that 
it
equals
the 
support of a GP confidence interval
in 
direction~$h$ for a fixed vector of noise parameters, we can derive the following ellipsoidal uncertainty bound.
\begin{corollary}
  \label{thm:suboptimal_bound_multivariate}
  Let \cref{ass:bounded_noise,ass:rkhs_norm_f} hold. Then, $\forall \sigma \in \Rge^{n_{\mathrm{con}}}$
  \begin{align}
    \label{eq:suboptimal_bound_multivariate}
    \| \ftr(x_{N+1}) - \fmu(x_{N+1}) \|_{\Sigma_\sigma(x_{N+1})^{-1}} \leq \beta_\sigma.
  \end{align}
\end{corollary}
\begin{proof}
  Using
  $h = \Sigma_\sigma(x_{N+1})^{-1} (\ftr(x_{N+1}) - \fmu(x_{N+1}))$, the assertion 
  is obtained 
  by
  rearranging \cref{thm:optimal_bound}, point~i).
\end{proof}
An essential advantage of 
the 
duality-based,
\emph{unconstrained} bound formulation
in \cref{thm:optimal_bound}, i) and \cref{thm:suboptimal_bound_multivariate}
is their straightforward integration into downstream optimization tasks: the dual (noise) parameters~$\sigma$ can be simultaneously optimized alongside the primary objective, tightening the uncertainty bound as necessary for the downstream task while guaranteeing safe predictions at each iteration. 

Invexity
in statement iii) is derived from the unique mapping between the original dual variables and the noise parameters~$\sigma$, and the convexity of the original dual function.
It guarantees 
that any local minimizer in the open domain~$\Rge^{n_{\mathrm{con}}}$ is globally optimal, 
rendering the bound attractive for scalable gradient-based optimization. 
Note, however, that the invexity property does not guarantee attainment of the optimizer in the open domain~$\Rge^{n_{\mathrm{con}}}$: indeed, the optimal bound in~\cref{eq:inf_opt_result} might be attained only in the limit for some $\sigma_j \rightarrow 0$ or $\sigma_j \rightarrow \infty$.

\subsubsection*{Illustrative example}

\cref{fig:illustrative_example} illustrates \cref{thm:optimal_bound} for a simple example,
where the proposed uncertainty bound is computed for a random unknown function $\ftr \in \Hkf$ with the squared-exponential kernel $\kf(x, x') = \exp(-\| x - x' \|_2^2)$, an RKHS-norm bound $\Gamma_f^2 = 1$ and $N=2$ training points corrupted by point-wise bounded noise realizations $| w_j | \leq 0.2$ (equivalent to special case 1 of \cref{ass:bounded_noise} with $\Gamma_{w,j} = 0.2$).
The optimal value of the dual function coincides with the optimal uncertainty bound at each respective test point due to strong duality;
additionally,
weak duality ensures that it also provides a conservative bound for the whole domain.
The incurred conservatism 
depends on the value of the noise parameters~$\sigma$:
for $[\sigma]_{j=1}^{n_{\mathrm{con}}} \rightarrow \infty$, the 
uncertainty bound~\eqref{eq:optimal_bound_expression} converges to the prior bound, \mbox{$\overline{f}^{\sigma}_h(x) \rightarrow \Gamma_f \sqrt{h^\top k(x, x) h}$};
if any $\sigma_j \rightarrow 0$, 
it \emph{may} grow unbounded;
cf.~\cite[Sec.~3.3]{lahr_optimal_2025} for a detailed discussion.

\begin{figure}[h]
  \centering
  \includegraphics{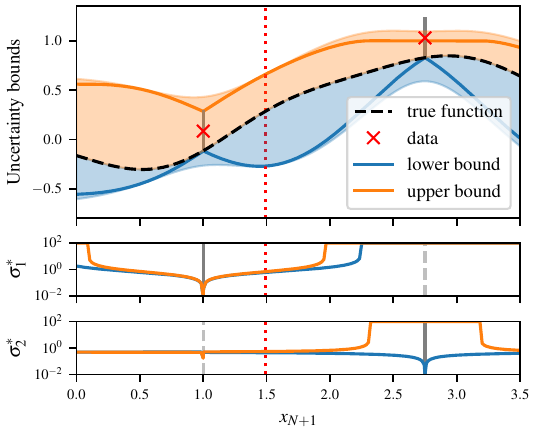}
  \caption{
  Illustrative example of proposed uncertainty bound.
  The top 
  plot shows
  the optimal uncertainty bounds (solid),
  as well as the corresponding dual functions (shaded), 
  evaluated for the optimal dual (noise) parameters~$\sigma^\star$ 
  at
  the 
  test point $x_{N+1} = 1.5$ (dotted red).
  The bottom two plots show the corresponding optimal value of the dual (noise) parameters $\sigma^\star$ for all test points.
  }
  \label{fig:illustrative_example}
\end{figure}

\section{Discussion}
\label{sec:discussion}

Next,
we compare the proposed bound with available kernel regression bounds. 
Since existing results are presented 
for scalar outputs, 
we restrict 
the discussion
to 
\mbox{$n_f = 1$}.

\subsubsection{Optimal deterministic bounds (optimization-based)}
\label{sec:discussion_optimal_scharnhorst}

The proposed bound generalizes the \emph{optimal} deterministic bounds developed by~\cite{scharnhorst_robust_2023} and \cite{lahr_optimal_2025} for point-wise bounded and energy-bounded noise realizations, 
which are both recovered as special cases~1 and~2 of \cref{ass:bounded_noise}, respectively.

For point-wise bounded noise,
assuming a uniform noise bound~$|w_i| \leq \bar{\Gamma}_w$ 
and a positive-definite kernel function~$\kf$,
\cite{scharnhorst_robust_2023} also provides a duality-based formulation of their bound, 
\begin{align}
  \label{eq:dual_scharnhorst}
    \overline{f}_h(x_{N+1}) = \inf_{\substack{\nu \in \mathbb{R}^{N}, \\ \lambda > 0}} \quad & 
    y^\top \nu + \bar{\Gamma}_{w} \| \nu \|_1 + \lambda \Gamma_f^2 \\
    & \hspace{-20ex} + \frac{1}{4 \lambda} \bigg( K^f_{N+1,N+1} + \| C \nu \|_{K^f_{1:N,1:N}}^2- 2 K^f_{N+1,1:N} C \nu \bigg). \notag
\end{align}
Compared with the proposed dual objective in \cref{thm:optimal_bound}, which is invex and smooth,
the dual objective~\eqref{eq:dual_scharnhorst} is convex and non-smooth.
To address non-smoothness, \cite[Alg.~1]{scharnhorst_robust_2023} proposes an alternating-optimization procedure, 
iterating
closed-form optimization of the scalar multiplier~$\lambda$ 
and 1-norm-constrained quadratic minimization of~$\nu \in \mathbb{R}^{N}$.
While 
this algorithm
is found to be effective for iterating towards the optimal uncertainty bound (see \cref{sec:application}), 
it remains difficult to integrate into downstream (optimization) tasks as it still involves inequality constraints.
In contrast, the proposed dual objective~\eqref{eq:optimal_bound_expression} 
is straightforwardly
integrated into optimization pipelines, optimizing~$\sigma$ as part of the downstream objective.

For energy-bounded noise, 
\cref{thm:optimal_bound} 
encompasses
\cite[Theorem~1]{lahr_optimal_2025},
which equivalently
expresses
the optimal uncertainty bound in terms of unconstrained optimization of the (scalar) noise parameter~$\sigma$.
This paper 
extends
these results
by additionally showing invexity of the dual function~\eqref{eq:optimal_bound_expression}. 
Moreover, the duality-based viewpoint of the upper bound
simplifies the proof, cf.~\cite[Appendix~C]{lahr_optimal_2025}, 
while 
explaining
the potential conservatism of existing suboptimal, closed-form bounds as discussed in the following.

\subsubsection{Suboptimal deterministic bounds (closed-form)}

Available closed-form, suboptimal bounds in the bounded-noise setting can be classified into two 
types. 

First, 
utilizing the Cauchy-Schwarz inequality,
the works \cite[Lemma~2]{hashimoto_learning-based_2022-1}, \cite[Lemma~2]{yang_kernel-based_2024}
obtain a bound of the form
\begin{align}
    \label{eq:suboptimal_bound_scalar}
  | \ftr(x_{N+1}) - \fmu[\sigma](x_{N+1}) | \leq \beta_{\sigma} \sqrt{\Sigma_{\sigma}(x_{N+1})}
\end{align}
for 
a \emph{fixed} value of 
$\sigma = \bar{\sigma} \mathds{1}_N$
and 
a \emph{uniform} noise bound $|w_i| \leq \bar{\Gamma}_w$ for all $i \in \mathbb{I}_1^N$.
In this case, the scaling factor $\beta_\sigma$ in \cref{eq:gp_definitions} simplifies to 
\mbox{$\beta_\sigma^2 = \Gamma_f^2 + N \bar{\Gamma}_w^2 \bar{\sigma}^{-2} - \| y \|_{\hat{K}_{ \sigma }^{-1}}^2$},
with $\bar{\sigma}^2 = \bar{\Gamma}_w^2$ and $\bar{\sigma}^2 = N \bar{\Gamma}_w^2$ 
for \cite{hashimoto_learning-based_2022-1} and \cite{yang_kernel-based_2024}, respectively. 
\cref{thm:suboptimal_bound_multivariate} 
recovers these bounds
for 
point-wise bounded 
noise (special case 1 of \cref{ass:bounded_noise}).
Their
conservatism
can 
thus
be explained by 
their equivalence with
the dual function~\eqref{eq:optimal_bound_expression}, evaluated for a fixed set of dual (noise) parameters.

Second, 
in \cite{maddalena_deterministic_2021,scharnhorst_robust_2023,reed_error_2025}, 
suboptimal, closed-form uncertainty bounds are obtained by
separating the 
prediction error
in an interpolation- and a noise-error term.
For example, 
the bound formulated in~\cite[Theorem~2]{reed_error_2025} reads as
\begin{align}
  \label{eq:bound_reed}
  \begin{aligned}
    | \ftr(x_{N+1}) - \fmu[\sigma](x_{N+1}) | \leq &\> \beta^f_{\max} \sqrt{\Sigma_{\sigma}(x_{N+1})} \\
    & + \bar{\Gamma}_w \| \hat{K}_\sigma^{-1} K^f_{1:N,N+1} \|_1,
  \end{aligned}
\end{align}
where 
$\beta^f_{\max} \doteq \sqrt{\Gamma_f^2 - \min_{| w_i | \leq \bar{\Gamma}_w} \| y - w \|_{\hat{K}_\sigma^{-1}}^2 }$.
This formulation 
generalizes
those in \cite[Theorem~1]{maddalena_deterministic_2021},~\cite[Proposition~3]{scharnhorst_robust_2023} 
in 
that it allows for an arbitrary noise parameter~$\bar{\sigma} > 0$, with $\sigma = \bar{\sigma} \mathds{1}_N$.
In contrast, 
the preceding works 
can only deal with~$\bar{\sigma} = 0$ in the interpolation term 
(which comes with the additional 
 restriction 
to positive-definite kernels and unique training input locations), 
and bound the difference to the mean prediction~$\fmu[\sigma]$ separately using the triangle inequality.
Despite the additional conservatism incurred by the separation of errors, 
one advantage of these bounds is that 
the uncertainty bound stays bounded as $\bar{\sigma} \rightarrow 0$. 
Still, the evaluation of~$\beta^f_{\max}$ requires solving a quadratic program comparable in complexity to~\eqref{eq:dual_scharnhorst} whenever the data set is modified.

\subsubsection{Finite-dimensional hypothesis spaces}
Since our theoretical analysis applies to general  positive-semidefinite kernels (cf. \cref{ass:rkhs_norm_f}), it applies to the classical special case of linear regression with a finite-dimensional hypothesis space, i.e., $k(x,x')= \Phi(x) \Phi(x')^\top$ for some features $\Phi:\mathbb{R}^{n_x} \rightarrow \mathbb{R}^{n_f \times r}$. 
In terms of application to control, 
this can, e.g., be used for estimating (non-falsified) nonlinear system dynamics from noisy measurements, where the noise at each time step is bounded by a quadratic constraint~\eqref{eq:noise_constraints}, cf.~\cite[Sec.~5]{martin_guarantees_2023}. 
In the special case of a single quadratic 
noise constraint,
our result recovers classical analytical regression bounds under energy-bounded noise~\cite{fogel_system_1979}; see also~\cite[Sec.~4.1]{lahr_optimal_2025} for a discussion.

\subsubsection{Probabilistic bounds} 
The proposed bound shares the same structure with many probabilistic results, where~\eqref{eq:suboptimal_bound_scalar} holds with some user-defined confidence $1-\delta$, and measurement noise realizations are assumed to be \emph{conditionally zero-mean $R$-sub-Gaussian}~\cite{fiedler_practical_2021}. This different setting is reflected in the expressions for the scaling factor $\beta_\sigma$, 
which commonly feature both the variance proxy $R$ and the parameter $\delta$ enforcing the desired confidence level, cf., e.g.,~\cite[Eq.~(7)]{fiedler_safety_2024-1}. 
The
main 
limitation
of these bounds 
is
their 
strong assumption on the measurement noise, 
preventing them
from rigorously handling biased, correlated and adversarial disturbances---all scenarios encompassed by the proposed deterministic, distribution-free bounds. 
Moreover, they can handle only increasing training data-set sizes: as such, differently from the proposed approach, they cannot be applied to subset-of-data strategies to improve scalability.
Further discussions can be found in~\cite[Section~VI,~Example~1]{scharnhorst_robust_2023} and \cite[Section~4.3]{lahr_optimal_2025}.

\phantom{This fixes erroneous blue text}

\section{Application example}
\label{sec:application}

\begin{figure}
  \centering
  \includegraphics[trim=0 30 0 50, clip]{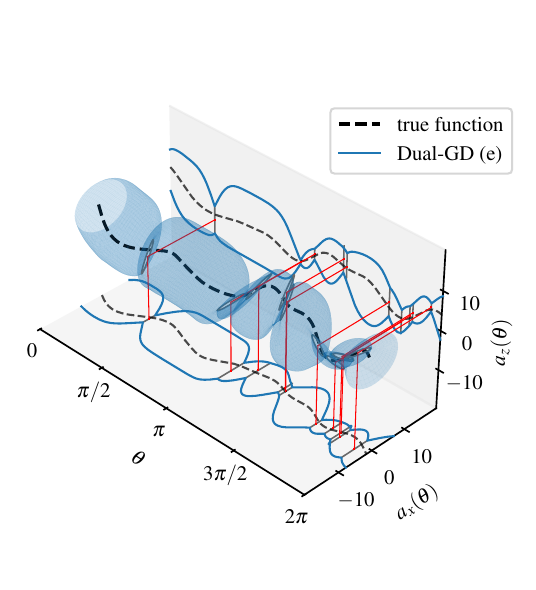}
  \caption{
    Proposed multivariate uncertainty bound for quadrotor example with $n_{\mathrm{data}} = 10$ training points.
    The latent function (dashed black) is tightly bounded by the optimal uncertainty bounds 
    evaluated for both output dimensions
    (\cref{thm:optimal_bound});
    the multivariate ellipsoidal tube is generated using $\sigma$-values corresponding to the optimal upper bound in $x$-direction (\cref{thm:suboptimal_bound_multivariate}). 
    The (projected) data points and (projected) uncertainty bounds are shown in red and gray, respectively. 
  }
  \label{fig:quadrotor}
\end{figure}

We compare the proposed bound with existing bounds using an 
example inspired by dynamics learning in control:
estimation of the residual acceleration $\ftr(x) = (a_x(\theta), a_z(\theta)) \in \mathbb{R}^{2}$ of a two-dimensional quadrotor~$(n_f = 2)$ in body frame as a function of its tilt angle~$\theta = x \in [0, 2\pi]$ $(n_x = 1)$. 
The 
$n_{\mathrm{data}} \in \{100, 1000\}$
noisy measurements of its accelerations $a_x, a_z$ 
are 
corrupted by an unknown wind force.
The wind force
is 
bounded by
a 2D ellipsoid
in the global coordinate frame,
modeling a direction-dependent
bound on the maximum wind magnitude, with stronger winds parallel to the ground\footnotemark.
This results
in tilt-angle-dependent, ellipsoidal noise bounds $P_j^w(\theta)$, $j \in \mathbb{I}_1^{n_{\mathrm{data}}}$,
each of which bounds $n_f$~measurements jointly in the $x$- and $z$-coordinate,
for a total of $N = n_f n_{\mathrm{data}}$ measurements.
The components of each measurement are extracted by defining 
$c_i$ accordingly as unit vectors.
We compute the optimal 
uncertainty bounds 
in 
$x$- and $z$-direction 
by setting $h$ as the corresponding 
unit vector,
leading to a box-shaped uncertainty set.
\cref{fig:quadrotor} depicts this 
setup
for $n_{\mathrm{data}} = 10$ measurements;
\footnotetext{%
In general, worst-case uncertainty bounds according to \cref{ass:bounded_noise} 
can be determined using domain knowledge, 
or may be estimated
from data~\cite{lauricella_set_2020}. 
}%
the full source code and implementation details are available
online\footnote{%
  Code available at \href{https://gitlab.ethz.ch/ics/bounded-rkhs-bounds-duality}{https://gitlab.ethz.ch/ics/bounded-rkhs-bounds-duality} or at \doi{10.3929/ethz-c-000801880}.%
}
and 
\iftoggle{arxiv}{in \cref{sec:app_implementation}%
}{in~\cite[Appendix~D]{lahr_optimal_2026}%
}, 
respectively.

\newcolumntype{R}{>{\raggedleft\arraybackslash}p{3.5ex}}
\newcolumntype{S}{>{\raggedleft\arraybackslash}p{6ex}}
\newcolumntype{T}{>{\raggedleft\arraybackslash}p{5ex}}
\newcolumntype{U}{>{\raggedleft\arraybackslash}p{6ex}}
\begin{table}
  \centering
  \caption{Experimental results for the quadrotor example
  }
  \begin{tabularx}{\columnwidth}{l|l|RRR|UUU}
  \toprule
  & \multirow{2}{*}{Method} & \multicolumn{3}{c|}{Suboptimality} & \multicolumn{3}{c}{Time [s]} \\
  & & min. & avg. & max. & min. & avg. & max. \\
  \midrule
  \multirow{6}{*}{\rotatebox[origin=c]{90}{$n_{\mathrm{data}} = 100$}} 
  & CVX-full (e) &   0.00 &  0.00 &  0.00 &   0.55 &  0.66 &  3.16 \\
  & CVX-full (p) &   0.01 &  0.22 &  0.70 &   0.13 &  0.19 &  0.27 \\
  & \cite{scharnhorst_robust_2023} (p) &   0.01 &  0.22 &  0.70 &   0.00 &  0.02 &  0.07 \\
  & \cite{reed_error_2025} (p) &   0.31 &  0.70 &  2.09 &   0.01 &  0.01 &  0.03 \\
  & Dual-GD (e) &   0.00 &  0.00 &  0.03 &   0.40 &  1.46 &  1.84 \\
  & Dual-GD (p) &   0.01 &  0.23 &  0.70 &   0.21 &  0.76 &  1.55 \\
  \midrule
  \multirow{6}{*}{\rotatebox[origin=c]{90}{$n_{\mathrm{data}} = 1000$}}  
  & CVX-full (e) &   0.00 &  0.00 &  0.00 & 289.66 &302.84 &326.85 \\
  & CVX-full (p) &   0.00 &  0.19 &  0.62 &  32.88 & 34.19 & 49.30 \\
  & \cite{scharnhorst_robust_2023} (p) &   0.00 &  0.19 &  0.62 &   0.92 &  2.68 & 26.83 \\
  & \cite{reed_error_2025} (p) &   0.46 &  0.71 &  1.30 &   0.04 &  0.05 &  0.11 \\
  & Dual-GD (e) &   0.03 &  0.15 &  0.39 &   3.45 &  5.10 & 16.70 \\
  & Dual-GD (p) &   0.18 &  0.37 &  0.79 &   0.76 &  1.07 &  2.27 \\
  \bottomrule
  \end{tabularx}
  \label{tab:quadrotor}
\end{table}

The following methods are compared in \cref{tab:quadrotor}: ``CVX-full'' solves the finite-dimensional version of Problem~\eqref{eq:inf_opt_infdim}, see \cref{sec:app_auxiliary}, using \texttt{CVXPY}~\cite{diamond_cvxpy_2016};
``\cite{scharnhorst_robust_2023}'' implements the iterative algorithm discussed in~\cref{sec:discussion_optimal_scharnhorst} in \texttt{CVXPY}~\cite{diamond_cvxpy_2016};
``\cite{reed_error_2025}'' computes the closed-form bound in~\cref{eq:bound_reed} with $\bar{\sigma} = \bar{\Gamma}_w$ set to the uniform noise error bound%
\footnote{Different tested choices $\bar{\sigma} \in \{10^{-3}, \bar{\Gamma}_w, 10^{3}\}$ lead to similar results, with $\bar{\sigma} = \bar{\Gamma}_w$ achieving the lowest suboptimality.};%
``Dual-GD'' optimizes Problem~\eqref{eq:inf_opt_result} with \texttt{GPyTorch}~\cite{gardner_gpytorch_2018} using the accelerated gradient descent method \texttt{Adam} with a fixed learning rate of 0.1.
We compare the conservatism resulting from 
the ellipsoidal ``(e)'' or
the corresponding smallest uniform point-wise ``(p)'' noise bound~$\bar{\Gamma}_w$ imposed by~\cite{reed_error_2025,scharnhorst_robust_2023}. 
This 
variant 
results in
a reduced number of \mbox{$N = n_{\mathrm{data}}$} data points since only measurements for the respective coordinate are utilized.
\cref{tab:quadrotor} shows the suboptimality, 
indicating the 
difference to the optimal 
bound ``CVX-full (e)'',
normalized by the difference of the prior uncertainty bound,
and the computation time (minimum, mean, and maximum).
The following stopping times are selected: for ``CVX-full'', until convergence; for \cite[Alg.~1]{scharnhorst_robust_2023}, until $10^{-2}$ suboptimality with respect to ``CVX-full (p)''; for ``Dual-GD'', after 100 gradient steps.

The results indicate that~\cite[Alg.~1]{scharnhorst_robust_2023} converges reliably and quickly to the optimal uncertainty bound assuming pointwise-bounded noise realizations, explaining the remaining difference to the optimal solution utilizing the ellipsoidal noise bound. 
The closed-form bound in \cite[Theorem~2]{reed_error_2025} naturally achieves the lowest computational time;
however, it also shows the highest 
conservatism.
For Dual-GD, 
reliable convergence to the optimal solution is observed for $n_{\mathrm{data}} = 100$. 
For $n_{\mathrm{data}} = 1000$,
convergence
is challenged by the high density of training points, 
which leads to more than 95\% of the optimal values $\sigma_j^\star \rightarrow \infty$ at the boundary of the domain for both noise assumptions. 
Nevertheless, 
it
effectively reduces the size of the uncertainty bound in the limited number of iterations, 
leading to competitive uncertainty estimates.

\section{Conclusions}

We have derived a tight, deterministic uncertainty bound for kernel-based worst-case uncertainty quantification.
The duality-based formulation 
extends 
kernel-based uncertainty bounds to 
more complex noise descriptions, 
and unifies
a range of existing results. 
Enabled by the unconstrained problem formulation, 
the 
numerical results 
highlight the 
effectiveness of the bound in 
reducing the size of the uncertainty envelope, %
motivating
its 
application
to downstream tasks such as Bayesian optimization or safe learning-based control.

\bibliographystyle{IEEEtran}
\bibliography{bibliography_amon_do_not_edit.bib}

\appendix

\subsection{Auxiliary results for \cref{thm:optimal_bound}}
\label{sec:app_auxiliary}

\begin{lemma}\emph{(cf.~\cite[Lemma~A.2]{lahr_optimal_2025})}
  \label{thm:findim}
  The optimal solution of~\eqref{eq:inf_opt_infdim} equals that of the finite-dimensional convex program%
  \begin{subequations}
    \label{eq:inf_opt_findim}
    \begin{align}
      \overline{f}_h(x_{N+1}) = \sup_{\substack{\theta \in\mathbb{R}^{r}}} \quad & h^\top \Phi_{N+1} \theta                                                                                    \\
      \mathrm{s.t.} \quad  
      & \| \theta \|_2^2 \leq \Gamma_f^2,                                                                        \\
      & \| y - C^\top \Phi_{1:N} \theta \|_{P^w_j}^2 \leq \Gamma_{w,j}^2,
    \end{align}
  \end{subequations}
  for $j \in \mathbb{I}_{1}^{n_{\mathrm{con}}}$.
  The feature matrix
  $\Phi_{1:N+1} \doteq [\Phi_i]_{i=1}^{N+1} \in \mathbb{R}^{n_f (N+1) \times r}$,
  with full column rank $r \leq n_f (N+1)$,
  composed of feature evaluations
  $\Phi_{i} \in\mathbb{R}^{n_f \times r}$
  at input locations~$x_i$,
  is given by 
  the Gram matrix factorization
  $K^f_{1:N+1,1:N+1} \doteq [k(x_i, x_j)]_{i,j=1}^{N+1} = \Phi_{1:N+1} \Phi_{1:N+1}^\top \in \mathbb{R}^{n_f (N+1) \times n_f (N+1)} $.
\end{lemma}

\begin{lemma}
  \label{thm:strong_duality}
  Slater's condition is satisfied for Problem~\eqref{eq:inf_opt_findim}. 
\end{lemma}
\begin{proof}
  Let
  $\ftrint$
  be the minimum-norm interpolant of $\ftr$ at the training- and test-inputs,
  i.e., the unique minimizer of
  $\min_{f \in \Hkf} \> \left\{ \| f \|_{\Hkf}^2 \> \middle | \> f(x_i) = \ftr(x_i), \> i \in \mathbb{I}_1^{N+1} \right\}$.
  By the representer theorem~\cite[Ch.~4.2]{scholkopf_learning_2002}, 
  $\exists \, \thetatr$ such that
  $\ftrint(x_{N+1}) = \Phi_{N+1} \thetatr$. Hence, 
  $\| \thetatr \|_2 = \| \ftrint \|_{\Hkf} \leq \| \ftr \|_{\Hkf} < \Gamma_f$, where the final inequality follows from \cref{ass:rkhs_norm_f}, establishing $\thetatr$ as a strictly feasible point of 
  Problem~\eqref{eq:inf_opt_findim}. 
  Thus, Slater's condition~\cite[Chap.~5.2.3]{boyd_convex_2004} is satisfied. 
\end{proof}%

\subsection{Proof of \cref{thm:optimal_bound}}
\label{sec:app_optimal_proof}

\subsubsection{Proof of statements i) and ii)}
\noindent
The optimal solution to the infinite-dimensional problem~\eqref{eq:inf_opt_infdim} coincides with the optimal solution of the finite-dimensional, convex program~\eqref{eq:inf_opt_findim} (cf.~\cref{thm:findim}).
The Lagrangian of~\eqref{eq:inf_opt_findim} is
\begin{align}
\label{eq:lagrangian_optimal}
\begin{aligned}
    \mathcal{L}(\theta, \lambda_{0:n_{\mathrm{con}}}) =& \> h^\top \Phi_{N+1} \theta - \lambda_0 \left( \| \theta \|_2^2 - \Gamma_f^2 \right) \\
    & \hspace{-4ex} - \sum_{j=1}^{n_{\mathrm{con}}} \lambda_j 
    \left(
    \| y - C^\top \Phi_{1:N} \theta \|_{P^w_j}^2 - \Gamma_{w,j}^2
    \right),
\end{aligned}
\end{align}
and
$d(\lambda_{0:n_{\mathrm{con}}}) = \sup_{\theta} \quad \mathcal{L}(\theta, \lambda_{0:n_{\mathrm{con}}})$,
its 
dual function.
Slater's condition
(\cref{thm:strong_duality})
implies 
that 
\begin{align}
\overline{f}_h(x_{N+1}) &= \inf_{\lambda_{0:n_{\mathrm{con}}} \in \Rgeq^{n_{\mathrm{con}}+1}} \quad d(\lambda_{0:n_{\mathrm{con}}}) \label{eq:dual_optimal_geq} \\
&= \inf_{\lambda_{0:n_{\mathrm{con}}} \in \Rge^{n_{\mathrm{con}}+1}} \quad d(\lambda_{0:n_{\mathrm{con}}}),
\label{eq:dual_optimal_ge}
\end{align}
where the first equality follows from strong duality and the second,
from $d(\lambda_{0:n_{\mathrm{con}}})$ being a proper, lower semi-continuous (closed), convex function~\cite[Corollary~7.5.1]{rockafellar_convex_1970}.
Hence, 
in the following 
we can restrict our analysis to $\lambda_j > 0$, $j \in \mathbb{I}_0^{n_\mathrm{con}}$.
Let
$d_0(\lambda_{1:n_{\mathrm{con}}}) = \inf_{\lambda_0 > 0} \> \sup_{\theta} \>  \mathcal{L}(\theta, \lambda_{0:n_{\mathrm{con}}})$
be the
``partially-minimized'' dual function.
We can write Problem~\eqref{eq:dual_optimal_ge} as
$\overline{f}_h(x_{N+1}) = \inf_{\lambda_{1:n_{\mathrm{con}}} \in \Rge^{n_{\mathrm{con}}}} \> 
d_0(\lambda_{1:n_{\mathrm{con}}})$
and
define
\begin{align}
    \mathcal{L}^\sigma(\theta, \lambda_0, \sigma) =& 
    \begin{aligned}[t]
    \> & h^\top \Phi_{N+1} \theta - \lambda_0 \left( \| \theta \|_2^2 - \Gamma_f^2 \right) \\
    & \hspace{-4ex} - \lambda_0 \left( 
    \| y - C^\top \Phi_{1:N} \theta \|_{P^w_\sigma}^2 - \sum_{j=1}^{n_{\mathrm{con}}} \frac{\Gamma_{w,j}^2}{\sigma_j^2}
    \right),
    \end{aligned}
    \label{eq:lagrangian_relaxed}
\end{align}
reparametrizing
\cref{eq:lagrangian_optimal},
$\mathcal{L}^\sigma(\theta, \lambda_0, \sigma) = \mathcal{L}(\theta, \lambda_{0:n_{\mathrm{con}}})$, 
for 
the dual (noise) parameters
$\sigma \doteq [\sigma_j^2]_{j=1}^{n_{\mathrm{con}}} = [\frac{\lambda_0}{\lambda_j}]_{j=1}^{n_{\mathrm{con}}} \in \Rge^{n_\mathrm{con}}$.
We analogously 
reparametrize
$
d_0^\sigma(\sigma) = \inf_{\lambda_0 > 0} \> \sup_{\theta} \>  \mathcal{L}^\sigma(\theta, \lambda_{0}, \sigma) 
$
and
recover the original optimal solution~\eqref{eq:dual_optimal_geq} as
\begin{align}
\label{eq:optimal_bound_using_d0_sig}
\overline{f}_h(x_{N+1}) = \inf_{\sigma \in \Rge^{n_\mathrm{con}}} \quad & d_0^\sigma(\sigma).
\end{align}
Finally, 
we derive
$d_0^\sigma(\sigma)$ 
in closed form
by noting that it is the 
dual solution 
for the 
primal problem
\begin{subequations}
\label{eq:inf_relax_findim}
\begin{align}
    \overline{f}_h^\sigma(x_{N+1}) = \sup_{\substack{\theta \in\mathbb{R}^{r}}} \quad 
    & 
    h^\top \Phi_{N+1} \theta                                                                                    \\
    \mathrm{s.t.} \quad 
    &
    \begin{aligned}[t]
    & \| \theta \|_2^2 + 
    \| y - C^\top \Phi_{1:N} \theta \|_{P^w_\sigma}^2
    \\
    & \leq \Gamma_f^2 + \sum_{j=1}^{n_{\mathrm{con}}} \frac{\Gamma_{w,j}^2}{\sigma_j^2}.\label{eq:inf_relax_findim_constr}
    \end{aligned}
\end{align}
\end{subequations}
The analytical solution to this convex program with linear cost and a single quadratic constraint is (see \cref{sec:app_relaxed_solution}) 
\begin{align}
\label{eq:relaxed_bound}
\overline{f}_h^\sigma(x_{N+1})
& = h^\top \fmu(x_{N+1}) + \beta_\sigma \sqrt{h^\top \Sigma_\sigma(x_{N+1}) h}.
\end{align}
For Problem~\eqref{eq:inf_relax_findim}, 
strong duality is established using the same strictly feasible point as in \cref{thm:strong_duality},
since every (strictly) feasible solution to~\eqref{eq:inf_opt_findim} is a (strictly) feasible solution to~\eqref{eq:inf_relax_findim}.
With $\overline{f}_h^\sigma(x_{N+1}) = d_0^\sigma(\sigma)$, 
inserting \cref{eq:relaxed_bound} in \cref{eq:optimal_bound_using_d0_sig} gives
statement i):
$h^\top \ftr(x_{N+1}) \leq \overline{f}_h(x_{N+1}) \leq \overline{f}_h^\sigma(x_{N+1})$;
statement ii) follows directly from~\cref{eq:optimal_bound_using_d0_sig}.
{\hfill $\qedsymbol$\par}

\subsubsection{Proof of statement iii)}
\label{sec:app_invexity_proof}

Consider the 
coordinate transformation 
$\psi(\lambda_0, \sigma) \doteq \lambda_0 \begin{bmatrix}
    1 & \sigma_1^{-2} & \ldots & \sigma_{n_{\mathrm{con}}}^{-2}
\end{bmatrix} = \lambda_{0:n_{\mathrm{con}}}$.
Noting that 
$d^\sigma_0(\sigma) = 
\inf_{\lambda_0 > 0} \> d^\sigma(\lambda_0, \sigma)
$
with
$d^\sigma(\lambda_0, \sigma) \doteq \sup_{\theta} \>  \mathcal{L}^\sigma(\theta, \lambda_{0}, \sigma)
= h^\top \fmu + \frac{1}{2 \lambda_0} \| h \|_{\Sigma_\sigma(x_{N+1})}^2 + \frac{\lambda_0}{2} \big( \Gamma_f^2 + \sum_{j=1}^{n_{\mathrm{con}}} \sigma_j^{-2} \Gamma_{w,j}^2 \big)$,
its 
unique
minimizer
is
$\lambda_0^\star(\sigma) = \| h \|_{\Sigma_\sigma(x_{N+1})} / 
(
    \Gamma_f^2 + \sum_{j=1}^{n_{\mathrm{con}}} \sigma_j^{-2} \Gamma_{w,j}^2
)^{-1/2}
$.
For all $\sigma \in \Rge^{n_{\mathrm{con}}}$,
since $\Sigma_\sigma(x_{N+1})$ is positive definite and $h \neq 0$,
it holds that $\lambda_0^\star(\sigma) > 0$.
Due to convexity of $d(\lambda_{0:n_{\mathrm{con}}})$ 
and non-singularity of 
the
Jacobian 
$\frac{\partial \psi}{\partial (\lambda_0, \sigma)}$ for all $(\lambda_0, \sigma) \in \Rge^{n_{\mathrm{con}}+1}$,
it holds that
\begin{align*}
    0 &= \left. \frac{\partial d_0^\sigma}{\partial \sigma} \right|_{\sigma^\star} 
    = 
    \left. \frac{\mathrm{d} d^\sigma(\lambda_0^\star(\sigma), \sigma)}{\mathrm{d} \sigma}\right|_{\sigma^\star} 
    = 
    \left. \frac{\mathrm{d} d(\psi(\lambda_0^\star(\sigma), \sigma))}{\mathrm{d} \sigma}\right|_{\sigma^\star} 
    \\
    &= 
    \left. \frac{\partial d(\lambda_{0:n_{\mathrm{con}}})}{\partial \lambda_{0:n_{\mathrm{con}}}} \right|_{\psi(\lambda_0^\star(\sigma^\star), \sigma^\star)} 
    \left. \frac{\partial \psi(\lambda_0, \sigma)}{\partial(\lambda_0, \sigma)} \right|_{\lambda_0^\star(\sigma^\star), \sigma^\star}
\end{align*}
if and only if 
$\left. \frac{\partial d(\lambda_{0:n_{\mathrm{con}}})}{\partial \lambda_{0:n_{\mathrm{con}}}} \right|_{\psi(\lambda_0^\star, \sigma^\star)} = 0$.
This means that
$\sigma^\star \in \Rge^{n_{\mathrm{con}}}$
locally minimizes
Problem~\eqref{eq:optimal_bound_using_d0_sig}
if and only if
$d(\psi(\lambda_0^\star, \sigma^\star)) = d_0^\sigma(\sigma^\star)$
is a global minimum
of 
\eqref{eq:optimal_bound_using_d0_sig}---equivalent to
invexity of
the function 
$d_0^\sigma(\sigma) \doteq \overline{f}_h^\sigma(\sigma)$. 
{\hfill $\qedsymbol$\par}

\subsection{Analytical solution of Problem~\eqref{eq:inf_relax_findim}}

\label{sec:app_relaxed_solution}

\begin{proof}[\unskip\nopunct]
  Problem~\eqref{eq:inf_relax_findim} is solved through a suitable reparametrization. First,
  consider~\eqref{eq:inf_relax_findim}
  and 
  let
  $h^\top \Phi_{N+1} \doteq q^\top$, $C^\top \Phi_{1:N} \doteq A^\top$ and $\Gamma_f^2 + \sum_{j=1}^{n_{\mathrm{con}}} \sigma_j^{-2} \Gamma_{w,j}^2 \doteq \Gamma^2$; additionally, define $S \doteq I + A P^w_\sigma A^\top$ with matrix square-root $S = S^{1/2} S^{1/2}$. 
  Finally, let $\thetamu \doteq A \left( A^\top A + \left( P^w_\sigma \right)^{-1} \right)^{-1} y = S^{-1} A P^w_\sigma y$ 
  yield
  the finite-dimensional parametrization of the GP posterior mean~$\fmu(x_{N+1}) = \Phi_{N+1} \thetamu$. 
  We rewrite
  \eqref{eq:inf_relax_findim_constr} as
  \begin{equation}\label{eq:relaxed_derivation_ellipsoid}
       \| \theta - \thetamu \|_S^2 
    + \| y \|_{\hat{K}_\sigma^{-1}}^2 \leq \Gamma^2,
  \end{equation}
  with
  $P^w_\sigma - P^w_\sigma A^\top S^{-1} A P^w_\sigma = \big( A^\top A + \left( P^w_\sigma \right)^{-1} \big)^{-1} = \hat{K}_\sigma^{-1}$
  entering \eqref{eq:gp_definitions}. 
  Using \cref{eq:relaxed_derivation_ellipsoid} and 
  defining
  $\xi \doteq S^{1/2} \left( \theta - \thetamu \right)$, 
  we 
  rewrite~\eqref{eq:inf_relax_findim} 
  as a linear problem with norm-ball constraint:%
  \begin{subequations}
    \begin{align}
      \overline{f}_h^\sigma(x_{N+1}) = \sup_{\substack{\xi \in\mathbb{R}^{r}}} \quad 
      & 
      q^\top \thetamu + q^\top S^{-1/2} \xi 
      \label{eq:app_relaxed_cost_xi}
      \\
      \mathrm{s.t.} \quad 
      &
      \begin{aligned}[t]
        & 
        \| \xi \|_2^2 \leq \Gamma^2 - \| y \|_{\hat{K}_\sigma^{-1}}^2.
      \end{aligned}
    \end{align}
  \end{subequations} 
  The unique optimal solution is 
  $\xi^\star = \sqrt{\Gamma^2 - \| y \|_{\hat{K}_\sigma^{-1}}^2} \frac{S^{-1/2} q}{\| q \|_{S^{-1}}}$,
  defining the value of the worst-case latent function at the test point: $f^\star(x_{N+1}) = \Phi_{N+1} \theta^\star = \Phi_{N+1} (\thetamu + S^{-1/2} \xi^\star)$. 
  With
  \begin{align*}
    \| q \|_{S^{-1}}^2 
    &= q^\top \left( I - A \left( A^\top A + (P^w_\sigma)^{-1} \right)^{-1} A^\top \right) q 
    \\
    &= h^\top \Sigma_\sigma(x_{N+1}) h,
  \end{align*}
  inserting $\xi^\star$ into \cref{eq:app_relaxed_cost_xi} leads to \cref{eq:relaxed_bound}.
  \par\nopagebreak\vspace{-1.0\baselineskip}\mbox{}
\end{proof}

\ifarxiv
\else
\vfill
\fi

\ifarxiv

\subsection{Application example: Implementation details}

\label{sec:app_implementation}

\subsubsection{Modeling setup}

Inspired by dynamics learning for control, this application example illustrates the application of the proposed uncertainty bound for learning the acceleration dynamics~$\ftr(x) = (a_x(x), a_z(x)) \in \mathbb{R}^{n_f}$, $n_f = 2$, of a quadrotor. 
For simplicity and illustration purposes, we consider the tilt angle~$\theta = x \in [0, 2\pi]$ of the quadrotor as scalar input variable, i.e., 
$n_x = 1$.
In the following, we provide a detailed description of the 
experimental setup. 

\subsubsection{Ground-truth function} 

The 
function $\ftr$
is randomly drawn from a RKHS corresponding to the periodic kernel 
\begin{align*}
  k(x, x') = \sigma_f \cdot \mathrm{exp}\left( - \frac{2}{\ell} \sin^2 \left( \frac{\pi}{p}(x - x') \right) \right),
\end{align*}
with output scale $\sigma_f \doteq 100$, length scale $\ell \doteq 0.1$, and period length $p \doteq 2 \pi$.
First, 
$n_{\bar{x}} \doteq 100$ random input locations 
are drawn from a uniform distribution,
$\bar{x}_k \sim \mathcal{U}([0, 2\pi])$,
with corresponding weights $\bar{\alpha}_{k,\mathrm{unif}} \sim \mathcal{U}([-1, 1]^{n_f})$, $k \in \mathbb{I}_1^{n_{\bar{x}}}$.
The function is then given by
\begin{align*}
  \ftr(x) \doteq \sum_{k=1}^{n_{\bar{x}}} k(x, x_k) \bar{\alpha}_{k},
\end{align*}
where $\bar{\alpha}_{k} \doteq \bar{\alpha}_{k,\mathrm{unif}} 
\left( \Gamma_f \middle/ \| \bar{\alpha}_{k,\mathrm{unif}} \|_{K^f_{1:N,1:N}} \right)
\in \mathbb{R}^{n_f}$
are rescaled weights such that $\| \ftr \|_{\Hk} = \Gamma_f \doteq 1$.

\subsubsection{Noise bounds}

The unknown wind force corrupting the measurements is assumed to be contained
within an ellipsoid with semi-axis lengths $\bar{w}_x \doteq 4.0$ and $\bar{w}_z \doteq 0.5$,
corresponding to a higher maximum wind force parallel to the ground~($x$-direction) than in the vertical direction~($z$-direction).
This leads to the ellipsoidal bound on the wind force 
\begin{align*}
  \| w^{\mathrm{wf}}_j \|_{P_{\mathrm{wind}}}^2 \leq 1, && P_{\mathrm{wind}} \doteq \begin{bmatrix}
    \frac{1}{\bar{w}_x^2} & 0 \\
    0 & \frac{1}{\bar{w}_z^2}
  \end{bmatrix}
\end{align*}
where $w^{\mathrm{wf}}_j \in \mathbb{R}^{2}$ denotes the vector-valued realization of the wind force \emph{in world frame},
and $j \in \mathbb{I}_1^{n_{\mathrm{data}}}$, the index of the vector-valued datum.
The wind force is converted to body frame by applying the 
rotation
\begin{align*}
  w^{\mathrm{bf}}_j = R(x_j) w^{\mathrm{wf}}_j, && R(x) = \begin{bmatrix}
    \cos(x) & - \sin(x) \\
    \sin(x) & \cos(x)
  \end{bmatrix},
\end{align*}
where $x_j$ denotes the corresponding tilt angle of the quadrotor. 
This results in
the ellipsoidal noise bound
\begin{align*}
  \| w^{\mathrm{bf}}_j \|_{P_j^{\mathrm{bf}}}^2 \leq 1, && P_j^{\mathrm{bf}} \doteq R(x_j) P_{\mathrm{wind}} R(x_j)^\top,
\end{align*}
for all $n_{\mathrm{data}}$ measurement locations~$x_j$, $j \in \mathrm{I}_1^{n_{\mathrm{data}}}$.
Finally, the matrices $P_j \in \mathbb{R}^{N \times N}$ are obtained by rewriting the above noise constraint in terms of the joint vector of $N = n_f n_{\mathrm{data}}$ noise realizations $w \doteq [ w^{\mathrm{bf}}_j ]_{j=1}^{n_{\mathrm{data}}} = [ w_i ]_{i=1}^N$,
such that $\| w^{\mathrm{bf}}_j \|_{P_j^{\mathrm{bf}}}^2 = \| w \|_{P_j}^2$ for all $j \in \mathrm{I}_1^{n_{\mathrm{data}}}$. 

\subsubsection{Noise sampling}

Random values for the wind force in body frame are generated as follows. 
First, rotated basis vectors $u_j \in \mathbb{R}^{2}$,
$u_j^\top = \begin{bmatrix}
  \sin(\beta_j) &
  \cos(\beta_j)
\end{bmatrix}^\top$,
$\beta_j \sim \mathcal{U}([0, 2 \pi])$,
as well as input locations
$x_j \sim \mathcal{U}([0, 2 \pi])$ 
are
are uniformly sampled for all $j \in \mathbb{I}_1^{n_{\mathrm{data}}}$. 
Second,
using 
the Cholesky decomposition 
$P_{\mathrm{wind}} = L_{\mathrm{wind}} L_{\mathrm{wind}}^\top$ 
of $P_{\mathrm{wind}}$,
the unit vectors are rescaled 
according to $w^{\mathrm{bf}}_j = R(x_j) L_{\mathrm{wind}}^{-\top} u_j$.

\subsubsection{Measurement model}

Given the sampled noise realizations and input locations, we obtain the measurements $y_j^{\mathrm{bf}}$ as
\begin{align*}
  y^{\mathrm{bf}}_{j} = \begin{bmatrix}
    c_x^\top \ftr(x_j) \\
    c_z^\top \ftr(x_j) \\
  \end{bmatrix}
  + 
  w^{\mathrm{bf}}_j,
\end{align*}
where $c_x \doteq \begin{bmatrix}
  1 & 0
\end{bmatrix}$
and $c_z \doteq \begin{bmatrix}
  0 & 1
\end{bmatrix}$
extract the $x$-
and $z$-component of the unknown function, respectively. 
The measurement model in the form of \cref{eq:data} is obtained again by stacking and re-indexing 
the output components $y \doteq [ y^{\mathrm{bf}}_j ]_{j=1}^{n_{\mathrm{data}}} = [ y_i ]_{i=1}^N$,
with appropriate assignment of $c_i$ equal to $c_x$ or $c_z$.

\subsubsection{Bound evaluation}
The optimal upper and lower uncertainty bounds in $x$- and $z$- direction are computed by setting 
\begin{align*}
  h_x = \pm \begin{bmatrix}
    1 & 0
  \end{bmatrix}, &&
  h_z = \pm \begin{bmatrix}
    0 & 1
  \end{bmatrix},
\end{align*}
as the vector $h$
in \cref{eq:inf_opt_result}
respectively.
This leads to the optimal containment interval $-\overline{f}_{-h}(x_{N+1}) \leq h^\top \ftr(x_{N+1}) \leq \overline{f}_h(x_{N+1})$ along both axis directions plotted in \cref{fig:quadrotor}.

\subsubsection{Implementation of existing methods}

For comparison with the uncertainty bounds proposed in \cite{scharnhorst_robust_2023} and \cite{reed_error_2025},
which are formulated in a scalar setting and for a uniform noise bound, the following modifications are made. 
First, to compute the bound along the $x$- and $z$-direction, the data set is subsampled such that only 
output measurements~$y_i$
of
the 
respective 
coordinate
are considered.
This leads to a reduced number of $N = n_{\mathrm{data}}$ scalar output measurements used to compute the upper and lower uncertainty bounds in the $x$- and $z$-direction, respectively (compared to $N = n_f n_{\mathrm{data}}$ output measurements for the multivariate bound).
The uniform uncertainty bound is set to $\bar{\Gamma}_w = \max \{ \bar{w}_x, \bar{w}_z \} = 4.0$.

\subsubsection{Suboptimality evaluation}

Denote by 
\begin{align*}
  \overline{f}_h^{\mathrm{prior}}(x_{N+1}) \doteq \Gamma_f \sqrt{h^\top k(x_{N+1}, x_{N+1}) h}
\end{align*}
the optimal prior uncertainty bound, given as the solution of Problem~\eqref{eq:inf_opt_infdim} for $N=0$, i.e., without any data (here, $\overline{f}_h^{\mathrm{prior}}(x_{N+1}) \equiv 10$).
Let $\overline{f}_h^{\mathrm{method}}(x_{N+1})$ denote the uncertainty bound computed with a given 
$\mathrm{method}$ according to \cref{tab:quadrotor}. 
The relative suboptimality
is defined as 
\begin{align*}
  \epsilon_h(x_{N+1}, \mathrm{method}) \doteq \frac{\overline{f}_h^{\mathrm{method}}(x_{N+1}) - \overline{f}_h^{\text{CVX-full (e)}}(x_{N+1})}{\overline{f}_h^{\mathrm{prior}}(x_{N+1}) - \overline{f}_h^{\text{CVX-full (e)}}(x_{N+1})},
\end{align*}
normalizing the size of the uncertainty bound based on its prior ($\epsilon_h = 1$) and optimal ($\epsilon_h = 0$) values.
\cref{tab:quadrotor} shows the relative suboptimality, average for $N_{\mathrm{test}} \doteq 50$ test points $x_{N+1}$ on an equidistant grid $\mathbb{X}_{\mathrm{grid}} \subset [0, 2 \pi]$ and for all considered directions $h \in \mathbb{H}_{\mathrm{grid}} \doteq \{ h_x, -h_x, h_z, -h_z \}$, i.e.,
\begin{align*}
  \mathrm{Suboptimality}(\mathrm{method}) \doteq \sum_{x \in \mathbb{X}_{\mathrm{grid}}} \sum_{h \in \mathbb{H}_{\mathrm{grid}}} \frac{\epsilon_h(x, \mathrm{method})}{\left| \mathbb{X}_{\mathrm{grid}} \right| \left| \mathbb{H}_{\mathrm{grid}} \right|},
\end{align*}
where the number of test points and bound directions is given as $\left| \mathbb{X}_{\mathrm{grid}} \right| = N_{\mathrm{test}}$ and $\left| \mathbb{H}_{\mathrm{grid}} \right| = 4$, respectively.

\pagebreak

\subsection{Generalized optimal uncertainty bound}
\label{sec:app_generalized_bounds}

We now generalize the problem set-up presented in \cref{sec:problem_setting}, enabling the incorporation of more prior information about the latent function, and extend \cref{thm:optimal_bound} to compute the optimal uncertainty bound in such a setting. We start by providing the new problem formulation that, instead of having a single RKHS-norm constraint as in \eqref{eq:inf_opt_infdim}, considers an intersection of ellipsoidal constraints on the \emph{joint} vector of latent-function parameters and noise realizations (\cref{sec:app_generalized_setup}). Next, we provide two classes of duality-based formulations. The first is a \emph{convex} one, and might be useful if bounds are to be computed using convex optimization solvers (\cref{sec:app_general_convex}). The second is a family of \emph{invex} formulations akin to \cref{thm:optimal_bound}, which rely on a more general class of coordinate transformations and can be beneficial to accelerate convergence of first-order methods for iterative refinement of uncertainty envelopes (\cref{sec:app_general_invex}). Then, we present two special cases of this new set-up. First, we specialize it to capture the problem stated in \cref{sec:problem_setting} and retrieve \cref{thm:optimal_bound} through an appropriate choice of coordinate transformation (\cref{sec:app_general_examples_recover_main_paper_result}). Second, we show how it can be used to obtain optimal uncertainty bounds for independent output components and component-wise RKHS-norm bounds for a general class of invexity-preserving transformations (\cref{sec:app_general_examples_componentwise_bounds}). Finally, we present 
two invexity-preserving transformations: 
the transformation used to recover \cref{thm:optimal_bound}, as well as a logarithmic transformation for optimizer-friendly rescaling of the dual variables
(\cref{sec:app_general_invex_trafos}). %

\subsubsection{Generalized problem setup}
\label{sec:app_generalized_setup}

Let us start with the generalized problem formulation. 
We consider uncertainty bounds that are obtained as the optimal solution to the convex program%
\begin{subequations}
  \label{eq:uncertainty_bound_ocp_findim}
  \begin{align}
    \overline{\alpha}_p \doteq \sup_{\alpha \in \mathbb{R}^{n_\alpha}} \quad & p^\top \alpha \\
    \mathrm{s.t.} \quad 
    & A^\top \alpha = y, 
    \label{eq:uncertainty_bound_ocp_findim_data} \\
    & 
    \| \alpha \|_{Q_j}^2 
    \leq \Gamma_j^2, \> \forall j \in \mathbb{I}_1^m.
    \label{eq:uncertainty_bound_ocp_findim_noise}
  \end{align}
\end{subequations}
with the linear cost defined by the vector~$p \in\mathbb{R}^{n_\alpha}$,
linear constraints defined through
$A^\top \in \mathbb{R}^{n \times n_\alpha}$,
and quadratic inequality constraints defined by 
$Q_j \in\mathbb{R}^{n_\alpha \times n_\alpha}$.
The finite-dimensional convex program~\eqref{eq:uncertainty_bound_ocp_findim} is a natural extension of Problem~\eqref{eq:inf_opt_findim}, 
which, through the representer theorem, determines the optimal solution to the infinite-dimensional optimization problem~\eqref{eq:inf_opt_infdim}.
In the following,
we directly derive 
duality-based 
formulations of the uncertainty bound 
based on
Problem~\eqref{eq:uncertainty_bound_ocp_findim}. 
In \cref{sec:app_general_examples_recover_main_paper_result,sec:app_general_examples_componentwise_bounds}
we then discuss two exemplary cases 
for which Problem~\eqref{eq:uncertainty_bound_ocp_findim} determines 
the optimal worst-case uncertainty bound 
starting from an infinite dimensional problem formulation
via
the representer theorem. 
We make the following assumptions.
\begin{assumption}
  \label{ass:app_noise}
  Let $Q_j \succeq 0$, $j \in \mathbb{I}_1^m$,
  with $\sum_{j=1}^{m} Q_j \succ 0$. 
\end{assumption}
\begin{assumption}
  \label{ass:app_strong_duality}
  Slater's condition is satisfied
  for Problem~\eqref{eq:uncertainty_bound_ocp_findim}. 
\end{assumption}

Throughout the following subsections, we will make use of the following definitions. 
First, let
\begin{align*}
  \Sump[\lambda] \doteq \sum_{j=1}^{m} \lambda_j Q_j, &&
  \Gram[\lambda] \doteq A^\top \Sump[\lambda]^{-1} A, &&
  \Gam[\lambda] \doteq \sum_{j=1}^{m} \lambda_j \Gamma_j^2,
\end{align*}
    which in turn define 
\begin{subequations}
  \begin{align}
    \Mu[\lambda] &= \Sump[\lambda]^{-1} A \Gram[\lambda]^{-1} y, \label{eq:app_optimal_invex_mean} \\
    \Sig[\lambda]
    &= 
      \Sump[\lambda]^{-1} - \Sump[\lambda]^{-1} A \Gram[\lambda]^{-1} A^\top \Sump[\lambda]^{-1} 
      \label{eq:app_optimal_invex_cov} \\
    \Beta[\lambda] &= \sqrt{\Gam[\lambda]^2 - \| y \|_{\Gram[\lambda]^{-1}}^2}.  \label{eq:app_optimal_invex_beta}
  \end{align}
\end{subequations}
To simplify notation for the invex coordinate transformations to be applied later, we treat these expressions as functions of the parameter vector $\lambda \in \Rge^m$, i.e., a different subscript here denotes the evaluation of the above expressions for a different parameter vector. 

\subsubsection{Convex formulation}
\label{sec:app_general_convex}

Using strong duality
the optimal solution to Problem~\eqref{eq:uncertainty_bound_ocp_findim} 
can be obtained from the following \emph{convex} dual formulation.
\begin{theorem}
  \label{thm:app_optimal_convex}
  Let \cref{ass:app_noise,ass:app_strong_duality} hold and define
  \begin{subequations}
    \label{eq:app_optimal_convex_expr}
    \begin{align}
      d_\alpha(\lambda) &= 
      p^\top \Mu[\lambda] + \frac{1}{2} \left( \| p \|_{\Sig[\lambda]}^2 +
      \Beta[\lambda]^2
      \right) 
      \label{eq:app_optimal_convex_similar_invex}
      \\
      &= 
      \frac{1}{2} 
      \left( 
        \left\| \begin{bmatrix}
            p \\ y
        \end{bmatrix} \right\|_{
          R_\lambda^{-1}
        }^2 
        + \Gam[\lambda]^2 
      \right), 
      \label{eq:app_optimal_convex_quadratic}
    \end{align}
  \end{subequations}
  with
  \begin{align*}
    R_\lambda &\doteq \begin{bmatrix}
      \Sump[\lambda] & A \\
      A^\top & 0
    \end{bmatrix}.
  \end{align*}
  Then, the optimal solution to~\eqref{eq:uncertainty_bound_ocp_findim} is given by
  \begin{align}
    \label{eq:app_optimal_convex}
    \overline{\alpha}_p = \inf_{\lambda \in \Rge^m} \quad d_\alpha(\lambda).
  \end{align}
  Moreover, the function $d_\alpha(\lambda)$ is convex.
\end{theorem}
\begin{proof}
  The Lagrangian of Problem~\eqref{eq:uncertainty_bound_ocp_findim} is given as
  \begin{align*}
    \mathcal{L}(\alpha, \lambda, \nu) 
    &= p^\top \alpha - \nu^\top (A^\top \alpha - y) - \sum_{j=1}^m \frac{\lambda_i}{2} (\alpha^\top Q_j \alpha - \Gamma_j^2) \notag \\ 
    &= p^\top \alpha - \nu^\top (A^\top \alpha - y) - \frac{1}{2} \alpha^\top \Sump[\lambda] \alpha + \frac{1}{2} \Gam[\lambda]^2,
  \end{align*}
  where we 
  rescaled the multipliers $\lambda_i$ by 
  $\frac{1}{2}$ 
  w.l.o.g.
  The corresponding \emph{convex} dual function is
  \begin{align}
    \label{eq:app_dual_def_alpha}
    d(\lambda, \nu) &= \sup_{\alpha \in\mathbb{R}^{n_\alpha}} \quad \mathcal{L}(\alpha, \lambda, \nu).
  \end{align}
  Due to Slater's condition (\cref{ass:app_strong_duality}), 
  it holds that
  \begin{align}
    \label{eq:app_dual_ge_wlog}
    \overline{\alpha}_p = \inf_{\substack{\lambda \in \Rgeq^m \\ \nu \in \mathbb{R}^{n}}} \> d(\lambda, \nu) = \inf_{\substack{\lambda \in \Rge^m \\ \nu \in \mathbb{R}^{n}}} \> d(\lambda, \nu),
  \end{align}
  where the first equality follows from strong duality and the second,
  from $d(\lambda, \nu)$ being a proper, lower semi-continuous (closed), convex function~\cite[Corollary~7.5.1]{rockafellar_convex_1970}.
  Hence, in the following we can restrict our analysis to $\lambda_j > 0$, $j \in \mathbb{I}_1^m$. 
  Rearranging the stationarity condition for Problem~\eqref{eq:app_dual_def_alpha}
  gives
  the primal optimizer~$\alpha^\star$ as a function of the dual variables~$\lambda, \nu$:%
  \begin{align*}
    0 &= \nabla_\alpha \mathcal{L}(\alpha, \lambda, \nu) = p - A \nu - \Sump[\lambda] \alpha \\
    \Leftrightarrow \alpha^\star &= \Sump[\lambda]^{-1} \left( p - A \nu \right).
  \end{align*}
  The dual function in \cref{eq:app_dual_def_alpha}
  is thus given by 
  \begin{align}
    d(\lambda, \nu) &= 
    \mathcal{L}(\alpha^\star, \lambda, \nu) \notag \\
    &= 
    \begin{aligned}[t]
      & \| p \|_{\Sump[\lambda]^{-1}}^2 - 2 p^\top \Sump[\lambda]^{-1} A \nu + \| A \nu \|_{\Sump[\lambda]^{-1}}^2 +  y^\top \nu \\
      & - \frac{1}{2} \| p - A \nu \|_{\Sump[\lambda]^{-1}}^2 + \frac{1}{2} \Gam[\lambda]^2 
    \end{aligned} \notag \\ 
    &= \frac{1}{2} \| p - A \nu \|_{\Sump[\lambda]^{-1}}^2 + y^\top \nu + \frac{1}{2} \Gam[\lambda]^2. \label{eq:app_dual_lam_nu}
  \end{align}
  We 
  eliminate the dual variables~$\nu$ 
  based on the stationarity condition for Problem~\eqref{eq:app_dual_ge_wlog}:
  \begin{align*}
    0 &= \nabla_\nu d(\lambda, \nu) = -A^\top \Sump[\lambda]^{-1} (p - A \nu) + y \\
    \Leftrightarrow \nu^\star &= 
    \Gram[\lambda]^{-1} 
    \left( A^\top \Sump[\lambda]^{-1} p - y \right).
  \end{align*}
  Inserting $\nu^\star$ into 
  \cref{eq:app_dual_lam_nu}
  requires calculation of
  \begin{align*}
    p^\top \Sump[\lambda]^{-1} A \nu^\star &= p^\top \Sump[\lambda]^{-1} A \Gram[\lambda]^{-1} \left( A^\top \Sump[\lambda]^{-1} p - y \right) \\
    &= \| A^\top \Sump[\lambda]^{-1} p \|_{\Gram[\lambda]^{-1} }^2 - y^\top \Gram[\lambda]^{-1} A^\top \Sump[\lambda]^{-1} p, \\
    \| A \nu^\star \|_{\Sump[\lambda]^{-1}}^2 &= \| A^\top \Sump[\lambda]^{-1} p - y \|_{\Gram[\lambda]^{-1}}^2, \\
    y^\top \nu^\star &= y^\top \Gram[\lambda]^{-1} A^\top \Sump[\lambda]^{-1} p - \| y \|_{\Gram[\lambda]^{-1}}^2.
  \end{align*}
  Thus, the expression for the dual function simplifies to
  \begin{align}
    d(\lambda, \nu^\star) &= 
    \begin{aligned}[t]
      & \frac{1}{2} \| p \|_{\Sump[\lambda]^{-1}}^2 - \| A^\top \Sump[\lambda]^{-1} p \|_{\Gram[\lambda]^{-1} }^2 \\
      & + \frac{1}{2} \| A^\top \Sump[\lambda]^{-1} p - y \|_{\Gram[\lambda]^{-1}}^2 - \| y \|_{\Gram[\lambda]^{-1}}^2 \\
      & + 2 y^\top \Gram[\lambda]^{-1} A^\top \Sump[\lambda]^{-1} p + \frac{1}{2} \Gam[\lambda]^2
    \end{aligned} \notag \\
    &= 
    \begin{aligned}[t]
      & \frac{1}{2} \| p \|_{\Sump[\lambda]^{-1}}^2 - \frac{1}{2} \| A^\top \Sump[\lambda]^{-1} p \|_{\Gram[\lambda]^{-1} }^2 \\
      & + y^\top \Gram[\lambda]^{-1} A^\top \Sump[\lambda]^{-1} p - \frac{1}{2} \| y \|_{\Gram[\lambda]^{-1}}^2 + \frac{1}{2} \Gam[\lambda]^2
    \end{aligned} \notag \\
    &= 
    \begin{aligned}[t]
      & 
      p^\top \Sump[\lambda]^{-1} A \Gram[\lambda]^{-1} y \\
      & + \frac{1}{2} \left( \| p \|_{\Sump[\lambda]^{-1} - \Sump[\lambda]^{-1} A \Gram[\lambda]^{-1} A^\top \Sump[\lambda]^{-1} }^2 + \Gam[\lambda]^2 - \| y \|_{\Gram[\lambda]^{-1}}^2 \right)
    \end{aligned} \notag \\
    &= p^\top \Mu[\lambda] + \frac{1}{2} \left( \| p \|_{\Sig[\lambda]}^2 + \Gam[\lambda]^2 - \| y \|_{\Gram[\lambda]^{-1}}^2 \right)
    \label{eq:app_dual_lam_nuopt_pre_invex_mu_vis} 
    \\
    &= 
    \frac{1}{2}
    \left( 
      \left\| \begin{bmatrix}
          p \\ y
      \end{bmatrix} \right\|_{
        R_\lambda^{-1}
      }^2 
      +
      \Gam[\lambda]^2
    \right)
    \notag
    \\
    &= d_\alpha(\lambda), \label{eq:app_dual_lam_nuopt_def_d_alpha}
  \end{align}
  where, in the last line, we have used the Schur complement
  \begin{align}
    R_\lambda^{-1} &\doteq
    \begin{bmatrix}
      \Sump[\lambda]^{-1} - \Sump[\lambda]^{-1} A \Gram[\lambda]^{-1} A^\top \Sump[\lambda]^{-1} & \Sump[\lambda]^{-1} A \Gram[\lambda]^{-1} \\
      \Gram[\lambda]^{-1} A^\top \Sump[\lambda]^{-1} & -\Gram[\lambda]^{-1}
    \end{bmatrix} \notag\\
    &= \left(
      \begin{bmatrix}
        \Sump[\lambda] & A \\
        A^\top & 0
      \end{bmatrix}
    \right)^{-1}. 
  \end{align}
  \cref{eq:app_optimal_convex} is established by inserting 
  \cref{eq:app_dual_lam_nuopt_def_d_alpha}
  into \cref{eq:app_dual_ge_wlog}. 
  Finally,
  since partial minimization preserves convexity of the dual function~\cite[Sec.~3.2.3]{boyd_convex_2004},
  it holds that $d_\alpha(\lambda)$ is convex. 
  This concludes the proof.
\end{proof}

\subsubsection{Invex formulation}
\label{sec:app_general_invex}

Starting from the convex dual function, 
in the following, we 
obtain an \emph{invex} dual formulation,
by applying a general coordinate transformation~$\psi$.
We recall that, in 
\cref{sec:app_invexity_proof},
a particular coordinate transformation \mbox{$\lambda = \psi(\lambda_0, \sigma)$}, 
specifically, the transformation introduced in \cref{thm:app_trafo_thm1},
has been utilized to establish invexity of the optimal uncertainty bound.
Here, we extend this argument to a general class of coordinate transformations
using
the following lemma. 
\begin{lemma}[\cite{craven1981duality}, cf.~\cite{mishra_invexity_2008}, p.~12]
  \label{thm:invexity_invertible_jacobian}
  Let $d: \mathbb{R}^{m} \rightarrow \mathbb{R}^{}$ be differentiable and convex and $\psi: \mathbb{R}^{m} \rightarrow \mathbb{R}^{m}$ be differentiable, with invertible Jacobian~$\nabla \psi^\top$.
  Then, 
  $\tilde{d}(\cdot) = d(\psi(\cdot))$
  is invex.
\end{lemma}
\cref{thm:invexity_invertible_jacobian} specifies general conditions for a coordinate transformation~$\psi$ to obtain an invex function $\tilde{d}(\cdot) = d(\psi(\cdot))$ from the convex dual function~$d$.
However, to obtain an ellipsoidal uncertainty bound akin to the Gaussian process-based formulation in \cref{thm:optimal_bound}, we require a specific kind of coordinate transformation,
namely, a
radial decomposition $\psi(\kappa, \phi) = \kappa \rho(\phi)$, where $\kappa \in \Rge$ 
denotes a uniform scaling of all coordinates,
and the function~$\rho$ maps the remaining $m-1$ coordinates~$\phi \in \Omega \subseteq \mathbb{R}^{m-1}$
to a $(m-1)$-dimensional hyper-surface in $\mathbb{R}^{m}$.
\cref{ass:invexity_coordinate_trafo} specifies all requirements for such a transformation to satisfy the conditions in \cref{thm:invexity_invertible_jacobian}.
\begin{assumption}
  \label{ass:invexity_coordinate_trafo}
  Let $\kappa_\lambda: \mathbb{R}^{m} \rightarrow \Rge$, $\phi_\lambda: 
  \mathbb{R}^{m} 
  \rightarrow \Omega$, 
  $\Omega \subseteq \mathbb{R}^{m-1}$,
  and $\rho: \Omega \rightarrow \Rge^m$
  be differentiable functions 
  satisfying
  $\kappa_\lambda(\lambda) \rho(\phi_\lambda(\lambda)) = \lambda$ 
  for all $\lambda \in \Rge^m$. 
  Additionally, let 
  the matrix 
  $\begin{bmatrix}
    \rho(\phi) & \frac{\partial \rho}{\partial \phi}
  \end{bmatrix} \in \mathbb{R}^{m \times m}$ be 
  invertible
  $\forall \phi \in \Omega$.
\end{assumption}
In \cref{sec:app_general_invex_trafos}, we describe two exemplary functions satisfying \cref{ass:invexity_coordinate_trafo}.
Using this coordinate transformation, 
the
invex uncertainty bound is 
derived from the
convex formulation in \cref{thm:app_optimal_convex} as follows.
\begin{theorem}
  \label{thm:app_optimal_invex}
  Let \cref{ass:app_noise,ass:app_strong_duality,ass:invexity_coordinate_trafo} hold and define
  \begin{align}
    \label{eq:app_optimal_invex_d_alpha}
    \tilde{d}_\alpha(\phi) = p^\top \Mu[\rho(\phi)] + \Beta[\rho(\phi)] 
    \sqrt{p^\top \Sig[\rho(\phi)] p}.
  \end{align}
  Then, the optimal solution to Problem~\eqref{eq:uncertainty_bound_ocp_findim} is given by
  \begin{align}
    \overline{\alpha}_p = \inf_{\phi \in \Omega} \quad & \tilde{d}_\alpha(\phi).
    \label{eq:app_optimal_invex}
  \end{align}
  Moreover, the function $\tilde{d}_\alpha(\phi)$ is invex.
\end{theorem}
\begin{proof}
  Analogous to the proof of \cref{thm:app_optimal_convex}, the optimal solution to Problem~\eqref{eq:uncertainty_bound_ocp_findim} is given by the solution to the convex dual problem
  \begin{align}
    \overline{\alpha}_p 
    &= 
    \begin{aligned}[t]
      \inf_{\lambda \in \Rge^m} \quad & d_\alpha(\lambda).
    \end{aligned} 
    \label{eq:app_inf_lambda_pre_invex}
  \end{align}
  Now, we 
  derive an equivalent invex problem formulation by 
  applying the coordinate transformation according to \cref{thm:invexity_invertible_jacobian}.
  By \cref{ass:invexity_coordinate_trafo},
  for any $\lambda \in \Rge^m$, there exist $\kappa = \kappa_\lambda(\lambda) \in \Rge$, $\phi = \phi_\lambda(\lambda) \in \Omega$, such that $\psi(\kappa, \phi) = \kappa \rho(\phi) = \lambda$.
  Due to the homogeneity of \cref{eq:app_optimal_invex_mean,eq:app_optimal_invex_cov,eq:app_optimal_invex_beta},
  applying this 
  coordinate transformation 
  results in $\Mu[\lambda] = \Mu[\rho(\phi)]$, $\Sig[\lambda] = \kappa^{-1} \Sig[\rho(\phi)]$, $\Beta[\lambda]^2 = \kappa \Beta[\rho(\phi)]^2$,
  i.e., we can 
  rewrite Problem~\eqref{eq:app_inf_lambda_pre_invex} as
  \begin{align}
    \overline{\alpha}_p 
    &=
    \begin{aligned}[t]
      &\inf_{\phi \in \Omega} \quad \inf_{\kappa \in \Rge} \quad \tilde{d}_\kappa(\kappa, \phi),
    \end{aligned}  
    \label{eq:app_inf_kappa}
  \end{align}
  with the dual function in transformed coordinates being
  \begin{align*}
    \tilde{d}_\kappa(\kappa, \phi) &\doteq \tilde{d}_\alpha(\rho(\kappa, \phi)) \\
    &= p^\top \Mu[\rho(\phi)] + \frac{1}{2 \kappa} \| p \|_{\Sig[\rho(\phi)]}^2 + \frac{\kappa}{2} \left( \Gam[\rho(\phi)]^2 - \| y \|_{M_{\rho(\phi)}^{-1}}^2 \right).
  \end{align*}
  Partial minimization for the scalar variable~$\kappa$
  results in
  \begin{align}
    \label{eq:app_d_alpha_eval_kappa_star}
    \tilde{d}_\alpha(\phi) \doteq \inf_{\kappa \in \Rge} \quad \tilde{d}_\kappa(\kappa, \phi) &= 
    \lim_{\kappa \rightarrow \kappa^\star(\phi)} \> \tilde{d}_\kappa(\kappa, \phi),
  \end{align}
  where $\tilde{d}_\alpha(\phi)$ is as defined in \cref{eq:app_optimal_invex_d_alpha}, and
  \begin{align}
    \kappa^\star(\phi) = \sqrt{\frac{\| p \|_{\Sig[\rho(\phi)]}^2}{\Gam[\rho(\phi)]^2 - \| y \|_{M_{\rho(\phi)}^{-1}}^2}}.
    \label{eq:app_kappa_opt}
  \end{align}
  Computing the limiting value of
  \cref{eq:app_d_alpha_eval_kappa_star} 
  based on 
  \cref{eq:app_kappa_opt}
  leads to \cref{eq:app_optimal_invex_d_alpha};
  \cref{eq:app_optimal_invex} then follows directly from \cref{eq:app_inf_kappa}. 
  This concludes 
  the first part of the proof. 

  To show invexity, 
  we first consider the function $\tilde{d}_\kappa(\kappa, \phi)$.
  Again,
  by \cref{ass:invexity_coordinate_trafo},
  for any $\lambda \in \Rge^{m}$, there exist $\kappa = \kappa_\lambda(\lambda) \in \Rge$, $\phi = \phi_\lambda(\lambda) \in \Omega$, such that 
  $\tilde{d}_\kappa(\kappa, \phi) = d_\alpha(\rho(\kappa, \phi)) = d_\alpha(\lambda)$,
  with $d_\alpha(\lambda)$ being the original convex dual function. 
  Additionally, the Jacobian of~$\rho$, 
  \begin{align*}
    \frac{\partial \rho}{\partial (\kappa, \phi)} = \begin{bmatrix}
      \rho(\phi) & \kappa \frac{\partial \rho}{\partial \phi}
    \end{bmatrix},
  \end{align*}
  is invertible by \cref{ass:invexity_coordinate_trafo}; 
  hence, $\tilde{d}_\kappa(\kappa, \phi)$ is invex by \cref{thm:invexity_invertible_jacobian}. 
  To show invexity of the partially-minimized function~\eqref{eq:app_d_alpha_eval_kappa_star},
  we consider two cases based on the value of $\kappa^\star(\phi)$. 
  To this end, note that from~\cref{eq:app_kappa_opt} it follows that $\kappa^\star(\phi) = 0$ if and only if $p \in \mathrm{Range}(A)$, i.e., $\| p \|_{\Sig[\rho(\phi)]}^2 = 0$.

  First, let $p \notin \mathrm{Range}(A)$, i.e.,
  $\kappa^\star(\phi) > 0$ for all $\phi \in \Rge^{m-1}$.
  Due to convexity of $d_\alpha(\lambda)$ 
  and non-singularity of 
  the
  Jacobian 
  $\frac{\partial \psi}{\partial (\kappa, \phi)}$ for all $(\kappa, \phi) \in \Rge^{n_{\mathrm{con}}+1}$,
  it holds that
  \begin{align*}
      0 &= \left. \frac{\partial \tilde{d}_\alpha}{\partial \phi} \right|_{\phi^\star} 
      = 
      \left. \frac{\mathrm{d} \tilde{d}_\kappa(\kappa^\star(\phi), \phi)}{\mathrm{d} \phi}\right|_{\phi^\star} 
      = 
      \left. \frac{\mathrm{d} d_\alpha(\psi(\kappa^\star(\phi), \phi))}{\mathrm{d} \phi}\right|_{\phi^\star} 
      \\
      &= 
      \left. \frac{\partial d_\alpha(\lambda)}{\partial \lambda} \right|_{\psi(\kappa^\star(\phi^\star), \phi^\star)} 
      \left. \frac{\partial \psi(\kappa, \phi)}{\partial(\kappa, \phi)} \right|_{\kappa^\star(\phi^\star), \phi^\star}
  \end{align*}
  if and only if 
  $\left. \frac{\partial d(\lambda)}{\partial \lambda} \right|_{\psi(\kappa^\star, \phi^\star)} = 0$.
  Hence,
  $\phi^\star \in \Rge^{n_{\mathrm{con}}}$
  locally minimizes
  \eqref{eq:app_optimal_invex}
  if and only if
  $d_\alpha(\psi(\kappa^\star(\phi^\star), \phi^\star)) = \tilde{d}_\alpha(\phi^\star)$
  is a global minimum
  of 
  \eqref{eq:app_optimal_invex},
  i.e.,
  the function~$\tilde{d}_\alpha(\phi^\star)$ is invex.

  Now, let $p \in \mathrm{Range}(A)$, i.e., $\kappa^\star(\phi) = 0$ for all $\phi \in \Rge^{m-1}$.
  In this case, we can write $p = A z$ for some vector $z \in \mathbb{R}^{n_\alpha}$. 
  Inserted into Problem~\eqref{eq:uncertainty_bound_ocp_findim}, this leads to 
  the optimal solution $p^\top \alpha^\star = z^\top A^\top \alpha^\star = z^\top y$
  being fully determined by the interpolation constraint~\eqref{eq:uncertainty_bound_ocp_findim_data}.
  In terms of the convex dual objective in \cref{eq:app_d_alpha_eval_kappa_star}, 
  the globally optimal solution is recovered 
  asymptotically as
  \begin{align*}
    \tilde{d}_\alpha(\phi) &= \lim_{\kappa \rightarrow \kappa^\star(\phi)} \> \tilde{d}_\kappa(\kappa, \phi) \\
    &= z^\top y + \lim_{\kappa \rightarrow 0} \> \frac{\kappa}{2} \left( \Gam[\rho(\phi)]^2 - \| y \|_{M_{\rho(\phi)}^{-1}}^2 \right) \\
    &= z^\top y.
  \end{align*}
  Since 
  $\frac{\partial \tilde{d}_\alpha(\phi)}{\partial \phi} = 0$ 
  and $\tilde{d}_\alpha(\phi)$ is the global minimum for all $\phi \in \Rge^{m-1}$, this establishes invexity of the function~$\tilde{d}_\alpha(\phi)$ in this case as well, concluding the proof.
\end{proof}

Next, 
we provide 
two
exemplary problem instances that the generalized bounds in \cref{thm:app_optimal_convex,thm:app_optimal_invex} provide the optimal uncertainty bound for.

\subsubsection{Special case: Recovering \cref{thm:optimal_bound}}
\label{sec:app_general_examples_recover_main_paper_result}

To illustrate how \cref{thm:app_optimal_invex} generalizes \cref{thm:optimal_bound}, 
we show how the former recovers the latter as a special case.
To this end, 
with some abuse of notation,
let us define the
analog expressions for the GP posterior mean and covariance in \cref{eq:gp_definitions} as
\begin{subequations}
  \label{eq:app_gp_def_ex_1}
  \begin{align}
    \fmu[\lambda] &\doteq K^{f,\lambda}_{N+1,1:N} C \gram[\lambda]^{-1} y, \\
    \covar[\lambda] &\doteq K^{f,\lambda}_{N+1,N+1} - K^{f,\lambda}_{N+1,1:N} C \gram[\lambda]^{-1} C^\top K^{f,\lambda}_{1:N,N+1}, \\
    \bet[\lambda] &\doteq \sqrt{\gam[\lambda]^2 - \| y \|_{\gram[\lambda]^{-1}}^2},
    \intertext{where the Gram matrix is denoted as}
    \gram[\lambda] &\doteq 
      C^\top K^{f,\lambda}_{1:N,1:N} C + \Bigg( \sum_{j=1}^{n_{\mathrm{con}}} 
      \lambda_j
    P^w_j \Bigg)^{-1}
  \end{align}
\end{subequations}
and, for \emph{this subsection}, we use the short-hand notation 
\begin{align}
  K^{f,\lambda}_{\mathbb{I},\mathbb{J}} \doteq \lambda_0^{-1} K^f_{\mathbb{I},\mathbb{J}}
  \label{eq:app_kernel_lambda_ex_1}
\end{align}
for the kernel matrix divided by the Lagrange multiplier corresponding to the RKHS-norm constraint~\eqref{eq:inf_opt_infdim_rkhs}.
For all equations above, 
we employ
the same subscript/superscript notation as for \cref{eq:app_optimal_invex_mean,eq:app_optimal_invex_cov,eq:app_optimal_invex_beta},
replacing the components of $\lambda$ in the expressions with the components of the utilized subscript (e.g.~$\rho(\phi)$).
Note that the only difference between the definitions in~\cref{eq:app_gp_def_ex_1} and those in \cref{eq:gp_definitions}
is the dependency on the free parameters: both definitions are equal when 
setting
$\lambda = \rho(\sigma) = \begin{bmatrix}
    1 & \sigma_1^{-2} & \ldots & \sigma_{n_{\mathrm{con}}}^{-2}
\end{bmatrix}$. 
This transformation will be the key step in establishing the equivalence between the general formulation in \cref{thm:app_optimal_invex} and \cref{thm:optimal_bound}.

First, we derive an
equivalent convex formulation of \cref{thm:optimal_bound} by defining the components of Problem~\eqref{eq:uncertainty_bound_ocp_findim} according to Problem~\eqref{eq:inf_opt_findim}, the equivalent finite-dimensional formulation of Problem~\eqref{eq:inf_opt_infdim}.
To simplify the derivation steps in the proof, we will thereby assume that the kernel $k$ is positive definite;
nevertheless, we highlight that the final result in \cref{thm:app_optimal_invex_ex_1_recover} recovers \cref{thm:optimal_bound} exactly, 
which is proven for positive-\emph{semidefinite} kernels.

\begin{corollary}[Convex formulation of \cref{thm:optimal_bound}]
  \label{thm:app_general_examples_recover_convex}
  Let \cref{ass:rkhs_norm_f,ass:bounded_noise} hold. Define 
  \begin{align}
    \overline{f}^\lambda_h(x_{N+1}) \doteq & \> h^\top \fmu[\lambda] + 
    \frac{1}{2} \left( \| h \|_{\covar[\lambda]}^2 + \bet[\lambda]^2 \right)
    \label{eq:app_general_examples_recover_convex}
  \end{align}
  Then, the optimal 
  uncertainty bound~\eqref{eq:inf_opt_infdim}
  is given by
  \begin{align}
    \overline{f}_h(x_{N+1}) &= \inf_{\lambda \in \Rge^{n_{\mathrm{con}} + 1}} \quad \overline{f}^\lambda_h(x_{N+1}).
  \end{align}
  Moreover, the function $\overline{f}^\lambda_h(x_{N+1})$ is convex.
\end{corollary}
\begin{proof}
  Define 
  \begin{align*}
    \alpha \doteq \begin{bmatrix}
      \alpha^f_1 & \ldots & \alpha^f_{N+1} & 
      w_1 & \ldots & w_N
    \end{bmatrix} \doteq \begin{bmatrix}
      \alpha^f 
      & w
    \end{bmatrix}.
  \end{align*}
  Then, the optimal solution to Problem~\eqref{eq:inf_opt_infdim} is equal to the optimal solution to Problem~\eqref{eq:uncertainty_bound_ocp_findim} with 
  the following definitions.
  First, the cost vector $p^\top$ 
  and equality-constraint matrix $A^\top = [a_i^\top]_{i=1}^n$, $n = N$,
  are given by
  \begin{align}
    p^\top \alpha &= 
    \begin{bmatrix}
      h^\top K^f_{N+1,1:N+1} & 0
    \end{bmatrix}
    \begin{bmatrix}
      \alpha^f \\ w
    \end{bmatrix}
    = h^\top f(x_{N+1}), 
    \label{eq:app_example_1_cost} \\
    a_i^\top \alpha &= 
    \begin{bmatrix}
      c_i^\top K^f_{i,1:N+1} & e_i
    \end{bmatrix} 
    \begin{bmatrix}
      \alpha^f \\ w
    \end{bmatrix}
    = c_i^\top f(x_i) + w_i.
    \label{eq:app_example_1_data} 
  \end{align}
  The latter equation gives $A^\top = \begin{bmatrix} C^\top K^f_{1:N,1:N+1} & I_N \end{bmatrix}$.
  Second, the $m = n_{\mathrm{con}} + 1$ constraint matrices are given by
  \begin{align*}
    Q_1 &= \begin{bmatrix}
      K^f_{1:N+1,1:N+1} & 0 \\
      0 & 0
    \end{bmatrix},
    &
    Q_{j+1} = \begin{bmatrix}
      0 & 0 \\
      0 & P_j
    \end{bmatrix},
  \end{align*}
  such that the inequality constraints~\eqref{eq:uncertainty_bound_ocp_findim_noise} read as
  \begin{align*}
    \| \alpha \|_{Q_1}^2 &= 
    \| \alpha^f \|_{K^f_{1:N+1,1:N+1}}^2  
    \hat{=} \> \| f \|_{\Hk}^2, 
    & 
    \Gamma_1^2 &= \Gamma_f^2 \\
    \| \alpha \|_{Q_{j + 1}}^2 &= 
    \| w \|_{P_j}^2, 
    &
    \Gamma_{j+1}^2 &= \Gamma_{w,j}^2,
  \end{align*}
  for all $j \in \mathbb{I}_1^{n_\mathrm{con}}$.
  By construction, \cref{ass:app_noise} is satisfied;
  \cref{ass:app_strong_duality}
  is established 
  analogous to \cref{thm:strong_duality}.
  Thus, the optimal solution to Problem~\eqref{eq:inf_opt_infdim} is given by the convex dual formulation in \cref{thm:app_optimal_convex}. 
  Simplifying the terms in the convex dual function~$d_\alpha(\lambda)$ using the definition in~\cref{eq:app_dual_lam_nuopt_pre_invex_mu_vis},
  we first calculate
  \begin{align*}
    \Sump[\lambda] &= \sum_{j=1}^{m} \lambda_j Q_j \\
    &= 
    \begin{bmatrix}
      \lambda_0 K^f_{1:N+1,1:N+1} & 0 \\
      0 & \sum_{j=1}^{n_{\mathrm{con}}} \lambda_j P^w_j
    \end{bmatrix}, \\
    \Gram[\lambda] &= A^\top \Sump[\lambda]^{-1} A \\
    &= \begin{aligned}[t]
      & C^\top K^f_{1:N,1:N+1} \left( \lambda_0 K^f_{1:N+1,1:N+1} \right)^{-1} K^f_{1:N+1,1:N} C \\
      & + \left( \sum_{j=1}^{n_{\mathrm{con}}} \lambda_j P^w_j \right)^{-1}
    \end{aligned} \\
    &= C^\top K^{f,\lambda}_{1:N,1:N} C + \left( \sum_{j=1}^{n_{\mathrm{con}}} \lambda_j P^w_j \right)^{-1} 
    = \gram[\lambda], 
  \end{align*}
  where we used that 
  \begin{align*}
      & K^f_{\mathbb{I},1:N+1} (K^f_{1:N+1,1:N+1})^{-1} K^f_{1:N+1,\mathbb{J}} \\
      & = \Pi_\mathbb{I} K^f_{1:N+1,1:N+1} (K^f_{1:N+1,1:N+1})^{-1} K^f_{1:N+1,1:N+1} \Pi_\mathbb{J} \\
      & = K^f_{\mathbb{I},\mathbb{J}}
  \end{align*}
  for appropriately defined selection matrices $\Pi_\mathbb{I}, \Pi_\mathbb{J}$.
  Inserting the above definitions into~\cref{eq:app_dual_lam_nuopt_def_d_alpha}, we obtain
  \begin{subequations}
    \begin{align}
      p^\top \Mu[\lambda] &= p^\top \Sump[\lambda]^{-1} A \Gram[\lambda]^{-1} y \notag \\
      &= \begin{aligned}[t]
        & \begin{bmatrix} 
          h^\top K^f_{N+1,1:N+1} (\lambda_0 K^f_{1:N+1,1:N+1})^{-1} & 0
        \end{bmatrix} \\
        & \cdot \begin{bmatrix}
          K^f_{1:N+1,1:N} C \\
          I_N
        \end{bmatrix}
        \gram[\lambda]^{-1} 
        y
      \end{aligned} \notag \\
      &= h^\top K^{f,\lambda}_{N+1,1:N} C \gram[\lambda]^{-1} y \notag \\
      &= h^\top \fmu[\lambda], \label{eq:app_example_1_thm_ingredients_mean} \\
      \| p \|_{\Sig[\lambda]}^2 
      &= 
      \| p \|_{\Sump[\lambda]^{-1} - \Sump[\lambda]^{-1} A \Gram[\lambda]^{-1} A^\top \Sump[\lambda]^{-1}}^2 \notag \\
      &= 
      \| h \|_{K^{f,\lambda}_{N+1,N+1} - K^{f,\lambda}_{N+1,1:N} C \gram[\lambda]^{-1} C^\top K^{f,\lambda}_{1:N,N+1}}^2 \notag \\
      &= \| h \|_{\covar[\lambda]}^2, \label{eq:app_example_1_thm_ingredients_cov} 
    \end{align}
  \end{subequations}
  which 
  leads to \cref{eq:app_general_examples_recover_convex}, 
  proving that
  $d_\alpha(\lambda) = \overline{f}_h^\lambda(x_{N+1})$ in this case. 
  Finally, by definition, $\overline{f}_h^\lambda(x_{N+1})$ is convex.
\end{proof}
Building upon the convex formulation, the invex formulation recovering \cref{thm:optimal_bound} is derived by applying a specific coordinate transformation. 
This is shown in the following corollary, the statement of which is equivalent to \cref{thm:optimal_bound}.
\begin{corollary}[Recovering \cref{thm:optimal_bound}]
  \label{thm:app_optimal_invex_ex_1_recover}
  Let \cref{ass:rkhs_norm_f,ass:bounded_noise} hold. 
  Define
  $\rho(\phi) \doteq \begin{bmatrix} 
    1 & \phi_1^{-2} & \ldots & \phi_{n_{\mathrm{con}}}^{-2}
  \end{bmatrix}^\top$ and
  \begin{align}
    \overline{f}_h^\phi(x_{N+1}) &\doteq h^\top \fmu[\rho(\phi)]+ \bet[\rho(\phi)] \sqrt{h^\top \covar[\rho(\phi)] h}.
  \end{align}
  Then, the optimal 
  uncertainty bound~\eqref{eq:inf_opt_infdim} 
  is given by
  \begin{align}
    \overline{f}_h(x_{N+1}) = \inf_{\phi \in \Rge^{n_{\mathrm{con}}}} \quad \overline{f}_h^\phi(x_{N+1}).
  \end{align}
  Moreover, the function $\overline{f}_h^\phi(x_{N+1})$ is invex.
\end{corollary}

\begin{proof}
  By \cref{thm:app_trafo_thm1}, $\rho(\phi)$ satisfies \cref{ass:invexity_coordinate_trafo}.
  Using the definitions of $m, p, A, Q_j, \Gamma_j$ as in the proof of \cref{thm:app_general_examples_recover_convex},
  \cref{ass:app_noise,ass:app_strong_duality} are satisfied by construction.
  Hence, we can apply \cref{thm:app_optimal_invex} to determine the optimal solution to Problem~\eqref{eq:inf_opt_infdim} via \cref{eq:app_optimal_invex}.
  Using the definition of $\rho(\phi)$, 
  \cref{eq:app_example_1_thm_ingredients_mean,eq:app_example_1_thm_ingredients_cov} 
  read as 
  \begin{align*}
    p^\top \Mu[\rho(\phi)] &= h^\top \fmu[\rho(\phi)],  \\
    \| p \|_{\Sig[\rho(\phi)]}^2 &= \| h \|_{\covar[\rho(\phi)]}^2. 
  \end{align*}
  Inserting these expressions into \cref{eq:app_optimal_invex} results in 
  $\tilde{d}_\alpha(\phi) = \overline{f}_h^\phi(x_{N+1})$, 
  which is invex according to \cref{thm:app_optimal_invex}.
\end{proof}

\subsubsection{Special case: Component-wise RKHS-norm bounds}
\label{sec:app_general_examples_componentwise_bounds}

Next, we 
show how the results in \cref{thm:app_optimal_convex,thm:app_optimal_invex} apply to a generalized problem setup, 
where each component of the unknown multivariate function~$\ftr$ is modeled independently, with a separate RKHS-norm bound. 

The analog to \cref{ass:rkhs_norm_f} in this setting is as follows.
\begin{assumption}
  \label{ass:app_rkhs_norm_f_diag}
  Let 
  $\ftr \in \mathcal{H}_k$ be an element of the Reproducing Kernel Hilbert Space (RKHS) $\Hk$, defined by a given positive, diagonal kernel function \mbox{$\kf: \mathbb{R}^{n_x} \times \mathbb{R}^{n_x} \rightarrow \mathbb{R}^{n_f} \times \mathbb{R}^{n_f}$},
  \begin{align*}
    k(x, x') = \mathrm{diag}(\{ k_l(x,x') \}_{l=1}^{n_f}),
  \end{align*}
  with component-wise, positive definite kernel functions $k_l:\mathbb{R}^{n_x} \times \mathbb{R}^{n_x} \rightarrow \mathbb{R}^{}$.
  Furthermore, 
  let the norm of $\ftr$
  be bounded in each component, i.e.,
  $\| f_l \|_{\mathcal{H}_{k_l}} < \Gamma_{f,l}$,
  for known constants $\Gamma_{f,l} > 0$, $l \in \mathbb{I}_1^{n_f}$.
\end{assumption}
To simplify the upcoming derivations, we assumed the component-wise kernel functions to be positive definite.
For the measurements, we again consider the same partial measurements as in \cref{eq:data}, with jointly bounded noise realizations as in \cref{ass:bounded_noise}, inducing a coupling between the otherwise independent output dimensions of the latent function. 
Similar to Problem~\eqref{eq:inf_opt_infdim}, the corresponding optimization problem to determine the worst-case realization of the unknown function along an arbitrary direction $h \in \mathbb{R}^{n_f}$ now reads as
\begin{subequations}
  \label{eq:app_inf_opt_infdim_diag}
  \begin{align}
    \overline{f}_h(x_{N+1}) = \sup_{\substack{f \in \Hkf                                        \\ w \in \mathbb{R}^{N}}} \quad & h^\top f(x_{N+1}) \\
    \mathrm{s.t.} \quad 
    &
    c_i^\top f(x_i) + w_i = y_i, \> i \in \mathbb{I}_1^N, \label{eq:app_inf_opt_infdim_data_diag} \\
    & w^\top P^w_j w \leq \Gamma_{w,j}^2, \label{eq:app_inf_opt_infdim_noise_diag} 
    \> j \in \mathbb{I}_1^{n_{\mathrm{con}}}, \\
    & \| f_l \|_{\mathcal{H}_{k_l}}^2 \leq \Gamma_{f,l}^2, \> l \in \mathbb{I}_1^{n_f}. \label{eq:app_inf_opt_infdim_rkhs_diag}
  \end{align}
\end{subequations}
By applying the representer theorem analogous to~\cite[Lemma~A.2]{lahr_optimal_2025}, an equivalent, finite-dimensional problem formulation can be derived.
The main difference between Problem~\eqref{eq:app_inf_opt_infdim_diag} and Problem~\eqref{eq:inf_opt_infdim} is given by the diagonal structure of the kernel matrix and the separate RKHS-norm constraints~\cref{eq:app_inf_opt_infdim_rkhs_diag}. 
To this end, for \emph{this subsection}, we define the Gaussian process-based expressions as in~\cref{eq:app_gp_def_ex_1},
but 
with the modified kernel matrix
\begin{align}
  \label{eq:app_kernel_ex_2_diag}
   K^{f,\lambda}_{\mathbb{I},\mathbb{J}} \doteq \sum_{l=1}^{n_f} \lambda_l^{-1} K^{f,l}_{\mathbb{I},\mathbb{J}}
\end{align}
instead of~\cref{eq:app_kernel_lambda_ex_1}. Here, $K^{f,l}_{\mathbb{I},\mathbb{J}}$ 
denote the
``masked'' kernel matrices
formed by evaluations of the kernel~$k^l$ corresponding to output dimension~$l$, such that $K^f_{\mathbb{I},\mathbb{J}} = \sum_{l=1}^{n_f} K^{f,l}_{\mathbb{I},\mathbb{J}}$.
Note that, despite the interleaved structure, the matrix $K^f_{\mathbb{I},\mathbb{J}}$ is a permuted block-diagonal matrix, i.e., $(K^f_{\mathbb{I},\mathbb{J}})^{-1} = \sum_{l=1}^{n_f} (K^{f,l}_{\mathbb{I},\mathbb{J}})^{-1}$; this will be used to simplify the final expressions later.

Using the redefined kernel matrix in \cref{eq:app_kernel_ex_2_diag}, the 
optimal convex uncertainty bound in this case 
leads to the same optimal uncertainty bound as in~\cref{thm:app_general_examples_recover_convex}.
\begin{corollary}[Convex formulation for \cref{ass:app_rkhs_norm_f_diag}]
  \label{thm:app_general_examples_diag_convex}
  Let \cref{ass:bounded_noise,ass:app_rkhs_norm_f_diag} hold. Define
  \begin{align}
    \overline{f}^\lambda_h(x_{N+1}) \doteq h^\top \fmu[\lambda] + \frac{1}{2} \left( \| h \|_{\covar[\lambda]}^2 + \bet[\lambda]^2 \right).
    \label{eq:app_general_examples_diag_convex}
  \end{align}
  Then, the optimal uncertainty bound~\eqref{eq:app_inf_opt_infdim_diag} is given by
  \begin{align}
    \overline{f}_h(x_{N+1}) &= \inf_{\lambda \in \Rge^{n_f + n_{\mathrm{con}}}} \quad \overline{f}^\lambda_h(x_{N+1}).
  \end{align}
  Moreover, the function $\overline{f}^\lambda_h(x_{N+1})$ is convex.
\end{corollary}
\begin{proof}
  Define 
  \begin{align*}
    \alpha \doteq \begin{bmatrix}
      \alpha^f_1 & \ldots & \alpha^f_{N+1} & 
      w_1 & \ldots & w_N
    \end{bmatrix} \doteq \begin{bmatrix}
      \alpha^f 
      & w
    \end{bmatrix},
  \end{align*}
  where $\alpha^f_i = [ \alpha^f_{i,l} ]_{l=1}^{n_f}$, for all $i \in \mathbb{I}_1^{N+1}$.
  Then, the optimal solution to Problem~\eqref{eq:app_inf_opt_infdim_diag} is equal to the optimal solution to Problem~\eqref{eq:uncertainty_bound_ocp_findim} with 
  the following definitions.
  The cost vector $p^\top$ 
  and equality-constraint matrix $A^\top = [a_i^\top]_{i=1}^n$, $n = N$,
  are given by
  \cref{eq:app_example_1_cost,eq:app_example_1_data}, respectively.
  Second, the $m = n_{\mathrm{con}} + n_f$ constraint matrices are given by
  \begin{align*}
    Q_l &= \begin{bmatrix}
      K^{f,l}_{1:N+1,1:N+1} & 0 \\
      0 & 0
    \end{bmatrix},
    &
    Q_{j+1} = \begin{bmatrix}
      0 & 0 \\
      0 & P_j
    \end{bmatrix},
  \end{align*}
  where $K^{f,l}_{\mathbb{I},\mathbb{J}} = [ E_l E_l^\top K^f_{1:N+1,1:N+1} E_l E_l^\top ]_{\mathbb{I},\mathbb{J}}$
  is a ``masked'' Gram matrix,
  defined using
  selection matrices $E_l$ 
  that extract all components of the weights $\alpha^f$ corresponding to the $l$-th output dimension, i.e., $E_l^\top \alpha^f = [\alpha^f_{i,l}]_{i=1}^{N+1} \doteq \alpha^f_l$.
  This leads to $m = n_f + n_{\mathrm{con}}$ constraints~\eqref{eq:uncertainty_bound_ocp_findim_noise} with
  \begin{align*}
    \| \alpha \|_{Q_l}^2 &= 
    \| \alpha^f \|_{K^{f,l}_{1:N+1,1:N+1}}^2  
    \hat{=} \> \| f_l \|_{\mathcal{H}_{k^l}}^2,
    &
    \Gamma_l^2 &= \Gamma_{f,l}^2 \\
    \| \alpha \|_{Q_{j + n_f}}^2 &= 
    \| w \|_{P_j}^2, 
    &
    \Gamma_{j+n_f}^2 &= \Gamma_{w,j}^2,
  \end{align*}
  for all $j \in \mathbb{I}_1^{n_{\mathrm{con}}}$ and $l \in \mathbb{I}_1^{n_f}$.
  By construction, \cref{ass:app_noise} is satisfied;
  \cref{ass:app_strong_duality}
  is established 
  analogous to \cref{thm:strong_duality}.
  Thus, the optimal solution to Problem~\eqref{eq:inf_opt_infdim} is given by the convex dual formulation in \cref{thm:app_optimal_convex}. 
  The derivation of the 
  terms for the convex dual function~\cref{eq:app_dual_lam_nuopt_def_d_alpha}
  in this case
  is 
  analogous to the derivations preceding~\cref{eq:app_example_1_thm_ingredients_mean,eq:app_example_1_thm_ingredients_cov},
  with the additional step of using the block-diagonal structure of the kernel matrices to derive that 
  \begin{align*}
    & \left( \sum_{l=1}^{n_f} K^{f,l}_{\mathbb{I},1:N+1} \right) \left( \sum_{l=1}^{n_f} \lambda_l K^{f,l}_{1:N+1,1:N+1} \right)^{-1} \left( \sum_{l=1}^{n_f} K^{f,l}_{1:N+1,\mathbb{J}} \right) \\
    & =  \sum_{l=1}^{n_f} \lambda_l^{-1} K^{f,l}_{\mathbb{I},\mathbb{J}} = K^{f,\lambda}_{\mathbb{I},\mathbb{J}}.
  \end{align*}
  Finally,
  we obtain~\cref{eq:app_example_1_thm_ingredients_mean,eq:app_example_1_thm_ingredients_cov}
  with the kernel matrices given by~\cref{eq:app_kernel_ex_2_diag},
  concluding the proof. 
\end{proof}

Based on the convex formulation of the uncertainty bound, 
an invex formulation can again be derived by using an invexity-preserving coordinate transformation satisfying \cref{ass:invexity_coordinate_trafo}.
We first summarize
this result
in the following corollary;
afterwards, 
we show 
how \cref{ass:invexity_coordinate_trafo} can be used to derive alternative coordinate transformations to formulate the optimal uncertainty bound.
\begin{corollary}[Invex formulation for \cref{ass:app_rkhs_norm_f_diag}]
  Let \cref{ass:bounded_noise,ass:app_rkhs_norm_f_diag} hold.
  Let $\rho(\phi): \Omega \rightarrow \mathbb{R}^m$ satisfy \cref{ass:invexity_coordinate_trafo} and define
  \begin{align}
    \overline{f}_h^\phi(x_{N+1}) &\doteq h^\top \fmu[\rho(\phi)]+ \bet[\rho(\phi)] \sqrt{h^\top \covar[\rho(\phi)] h}.
  \end{align}
  Then, the optimal 
  uncertainty bound~\eqref{eq:inf_opt_infdim} 
  is given by
  \begin{align}
    \overline{f}_h(x_{N+1}) = \inf_{\phi \in \Omega} \quad \overline{f}_h^\phi(x_{N+1}).
  \end{align}
  Moreover, the function $\overline{f}_h^\phi(x_{N+1})$ is invex.
\end{corollary}
\begin{proof}
  Analogous to \cref{thm:app_optimal_invex_ex_1_recover},
  we can apply \cref{thm:app_optimal_invex} to determine the optimal solution to Problem~\eqref{eq:app_inf_opt_infdim_diag} via \cref{eq:app_optimal_invex}.
  \cref{eq:app_example_1_thm_ingredients_mean,eq:app_example_1_thm_ingredients_cov} 
  read as 
  \begin{align*}
    p^\top \Mu[\rho(\phi)] &= h^\top \fmu[\rho(\phi)],  \\
    \| p \|_{\Sig[\rho(\phi)]}^2 &= \| h \|_{\covar[\rho(\phi)]}^2. 
  \end{align*}
  Inserting these expressions into \cref{eq:app_optimal_invex} results in 
  $\tilde{d}_\alpha(\phi) = \overline{f}_h^\phi(x_{N+1})$, 
  which is invex according to \cref{thm:app_optimal_invex}.
\end{proof}

\subsubsection{Invexity-preserving coordinate transformations}
\label{sec:app_general_invex_trafos}

In this subsection, we provide two invexity-preserving transformations that can be used to derive an optimal invex uncertainty bound from the convex formulation.
First, we consider the coordinate transformation underlying the invexity proof \cref{thm:optimal_bound} in \cref{sec:app_invexity_proof}; afterwards, we provide a logarithmic invexity-preserving transformation that can be numerically favorable to handle the commonly large variations in the size of the dual variables.

\begin{lemma}[Invex transformation in \cref{thm:optimal_bound}]
  \label{thm:app_trafo_thm1}
  \cref{ass:invexity_coordinate_trafo} is satisfied by
  \begin{align*}
    \rho(\phi) &= \begin{bmatrix}
      1 & \phi_1^{-2} & \ldots & \phi_{n_{\mathrm{con}}}^{-2}
    \end{bmatrix}^\top,
  \end{align*}
  with $\Omega = \Rge^{m-1}$.
\end{lemma}
\begin{proof}
  Consider \cref{ass:invexity_coordinate_trafo} and 
  define
  the 
  coordinate transformation 
  $\lambda = \psi(\kappa_\lambda(\lambda), \phi_\lambda(\lambda)) = \kappa_\lambda(\lambda) \rho(\phi_\lambda(\lambda))$, with
  \begin{align}
    \kappa_\lambda(\lambda) &= \lambda_0, \\
    \phi_\lambda(\lambda) &= \begin{bmatrix}
      \sqrt{\frac{\lambda_0}{\lambda_1}} & \ldots & \sqrt{\frac{\lambda_0}{\lambda_{n_{\mathrm{con}}}}}
    \end{bmatrix}^\top.
  \end{align}
  Since the triangular matrix 
  \begin{align*}
    \begin{bmatrix}
      \rho(\phi) & \frac{\partial \rho}{\partial \phi}
    \end{bmatrix} 
    &=
    \begin{bmatrix}
      1 & & & \\
      \phi_1^{-2} & -2 \phi_1^{-3} & & \\
      \vdots & & \ddots & \\
      \phi_{n_{\mathrm{con}}}^{-2} & & & -2 \phi_{n_{\mathrm{con}}}^{-3}
    \end{bmatrix}
  \end{align*}
  is invertible for all $\phi \in \Omega = \Rge^{m-1}$, \cref{ass:invexity_coordinate_trafo} is satisfied.
  
\end{proof}

\begin{lemma}[Log-transformation]
  \label{lem:app_trafo_log}
  \cref{ass:invexity_coordinate_trafo} is satisfied by
  \begin{align*}
    \rho(\phi) &= \begin{bmatrix}
      e^{\phi_1} & \cdots & e^{\phi_{m-1}} & e^{-\sum_{j=1}^{m-1} \phi_j}
    \end{bmatrix}^\top,    
  \end{align*}
  with 
  $\Omega = \mathbb{R}^{m-1}$.
\end{lemma}
\begin{proof}
  Consider \cref{ass:invexity_coordinate_trafo} and define 
  \begin{align*}
    \kappa_\lambda(\lambda) &= \left( \prod_{j=1}^{m} \lambda_j \right)^{1/m} = \exp \left( {\frac{1}{m}\sum_{k=1}^{m} \ln \lambda_k} \right), \\    
    [\phi_\lambda(\lambda)]_j &= \ln \lambda_j - \frac{1}{m}\sum_{k=1}^{m} \ln \lambda_k,
  \end{align*}
  to 
  obtain a transformation 
  \begin{align*}
    & \kappa_\lambda(\lambda) \rho(\phi_\lambda(\lambda)) = 
    \kappa_\lambda(\lambda) \begin{bmatrix}
      \frac{\lambda_1}{\kappa_\lambda(\lambda)} & \ldots & \frac{\lambda_{m-1}}{\kappa_\lambda(\lambda)} & \frac{\lambda_{m}}{\kappa_\lambda(\lambda)}
    \end{bmatrix} = \lambda.
  \end{align*}
  The expression for the last component $\rho_m(\phi_\lambda(\lambda))$ is established via the equality 
  $-\sum_{j=1}^{m-1} [\phi_\lambda(\lambda)]_j = - \sum_{j=1}^{m-1} \ln \lambda_j + \frac{m-1}{m} \sum_{j=1}^{m} \ln \lambda_j = \ln \lambda_m - \frac{1}{m} \sum_{j=1}^{m} \ln \lambda_j$.
  For all $\phi \in \mathbb{R}^{m-1} = \Omega$,
  since $\rho_j(\phi) = e^{\phi_j} > 0$ , we have $\rho(\phi) \in \Rge^m$.
  It remains to show that the matrix 
  \begin{align*}
    & \begin{bmatrix}
      \rho(\phi) & \frac{\partial \rho}{\partial \phi}
    \end{bmatrix} 
    = \\
    & \begin{bmatrix}
      e^{\phi_1} & e^{\phi_1}  & & \\
      \vdots & & \ddots & \\
      e^{\phi_{m-1}} & & & e^{\phi_{m-1}} \\
      e^{-\sum_{j=1}^{m-1} \phi_j} & -e^{-\sum_{j=1}^{m-1} \phi_j} & \ldots & -e^{-\sum_{j=1}^{m-1} \phi_j}
    \end{bmatrix} 
  \end{align*}
  is non-singular.
  This can be seen 
  by 
  transforming it into the triangular matrix
  \begin{align*}
    \begin{bmatrix}
      m & & & \\
      1 & 1 & & \\
      \vdots & & \ddots & \\
      1 & & & 1
    \end{bmatrix} 
  \end{align*}
  using elementary row operations (divide the $j$-th row by $\rho_j(\phi) > 0$ for all $j \in \mathbb{I}_1^m$, add the $m-1$ top rows to the last row, swap the last row into the top row).
  This 
  shows that \cref{ass:invexity_coordinate_trafo} is satisfied.
\end{proof}

\vfill

\fi

\pagebreak

\ifnotes

\subsection{Results for other data set sizes}

\begin{table}[h!]
  \centering
  \caption{Comparison of suboptimality and computation times for the quadrotor example with $n_{\mathrm{data}} = 10$ training points.}
  \begin{tabularx}{\columnwidth}{l|RRR|SSS}
  \toprule
  Method & \multicolumn{3}{c|}{Suboptimality} & \multicolumn{3}{c}{Time [s]} \\
  \midrule
  CVX-full (e) &   0.00 &   0.00 &   0.00 &   0.01 &   0.01 &   0.07 \\
  CVX-full (p) &   0.07 &   1.54 &   4.06 &   0.01 &   0.01 &   0.06 \\
  \cite[Alg.~1]{scharnhorst_robust_2023} (p) &   0.07 &   1.54 &   4.06 &   0.00 &   0.00 &   0.02 \\
  \cite[Theorem~2]{reed_error_2025} (p) &   0.66 &   4.36 &  16.79 &   0.01 &   0.01 &   0.02 \\
  Dual-GD (e) &   0.00 &   0.00 &   0.12 &   0.23 &   0.43 &   1.00 \\
  Dual-GD (p) &   0.08 &   1.54 &   4.06 &   0.01 &   0.32 &   1.15 \\
  \bottomrule
  \end{tabularx}
  \label{tab:quadrotor_10}
\end{table}

\begin{table}[h!]
  \centering
  \caption{Comparison of suboptimality and computation times for the quadrotor example with $n_{\mathrm{data}} = 100$ training points.}
  \begin{tabularx}{\columnwidth}{l|RRR|SSS}
  \toprule
  Method & \multicolumn{3}{c|}{Suboptimality} & \multicolumn{3}{c}{Time [s]} \\
  \midrule
  CVX-full (e) &   0.00 &   0.00 &   0.00 &   0.56 &   0.64 &   0.84 \\
  CVX-full (p) &   0.06 &   1.80 &   4.43 &   0.17 &   0.19 &   0.29 \\
  \cite[Alg.~1]{scharnhorst_robust_2023} (p) &   0.06 &   1.80 &   4.43 &   0.00 &   0.02 &   0.11 \\
  \cite[Theorem~2]{reed_error_2025} (p) &   2.45 &   5.43 &  10.76 &   0.01 &   0.01 &   0.09 \\
  Dual-GD (e) &   0.00 &   0.04 &   0.98 &   0.35 &   1.47 &   1.88 \\
  Dual-GD (p) &   0.08 &   1.85 &   4.49 &   0.24 &   0.76 &   1.55 \\
  \bottomrule
  \end{tabularx}
  \label{tab:quadrotor_100}
\end{table}

\subsection{Alternative proof without relaxed formulations}

Due to strong duality (cf.~\cref{thm:strong_duality}),
the optimal solution to~\eqref{eq:inf_opt_findim} is given as the solution of the dual problem:
\begin{align}
\underline{f}_h(x_{N+1}) = \inf_{\lambda_0, \ldots, \lambda_{n_{\mathrm{con}}} \geq 0} \quad & d(\lambda_{0:n_{\mathrm{con}}}) 
\end{align}
with the 
convex 
dual function
\begin{align}
d(\lambda_{0:n_{\mathrm{con}}}) = \sup_{\theta} \quad \mathcal{L}(\theta, \lambda_{0:n_{\mathrm{con}}})
\end{align}
and the Lagrangian given as
\begin{align}
\begin{aligned}
    \mathcal{L}(\theta, \lambda_{0:n_{\mathrm{con}}}) =& \> h^\top \Phi_{N+1} \theta - \frac{\lambda_0}{2} \left( \| \theta \|_2^2 - \Gamma_f^2 \right) \\
    & \hspace{-4ex} - \sum_{j=1}^{n_{\mathrm{con}}} \frac{\lambda_j}{2}
    \left(
    \| y - C^\top \Phi_{1:N} \theta \|_{P^w_j}^2 - \Gamma_{w,j}^2
    \right) \\
    =& \> q^\top \theta - \sum_{j=0}^{n_{\mathrm{con}}} \frac{1}{2} \left( \| y_j - A_j^\top \theta \|_{P^\lambda_j}^2 - \Gamma_j^2 \right)
\end{aligned}
\end{align}
where we redefined $\lambda_j \leftarrow \frac{\lambda_j}{2}$ (without loss of generality, as constraints can be scaled) and used the additional abbreviations 
\begin{align}
  q^\top &= h^\top \Phi_{N+1} \\
  A_0^\top &= \begin{bmatrix}
    0_{N-r} & I_r
  \end{bmatrix}^\top \\
  A_j^\top &= A^\top = C^\top \Phi_{1:N}, \\
  P^\lambda_0 &=  \lambda_0 I_N \\
  P^\lambda_j &= \lambda_j P_j^w \\
  y_0 &= 0_N \\
  y_j &= y \\
  \Gamma_0 &= \Gamma_f \\
  \Gamma_j &= \Gamma_{w,j}
\end{align}
Stationarity condition:
\begin{align}
  0 &= \nabla_\theta \mathcal{L} \\
  &= q + \sum_{j=0}^{n_{\mathrm{con}}} A_j P^\lambda_j \left( y_j - A_j^\top \theta \right) \\
  \Leftrightarrow \theta^\star &= \left( \sum_{j=0}^{n_\mathrm{con}} A_j P^\lambda_j A_j^\top \right)^{-1} \left( \sum_{j=0}^{n_\mathrm{con}} A_j P^\lambda_j y_j + q \right) \\
  &= \left( \lambda_0 I_r + A N_\lambda A^\top \right)^{-1} \left( A N_\lambda y + q \right) \\
  &= \Gram[\lambda]^{-1} \left( p_\lambda + q \right)
\end{align}
where 
\begin{align}
  N_\lambda &= \sum_{j=1}^{n_{\mathrm{con}}} P^\lambda_j = \sum_{j=1}^{n_{\mathrm{con}}} \lambda_j P^w_j \\
  M_\lambda &= \lambda_0 I_r + A N_\lambda A^\top = \lambda_0 I_r + \sum_{j=1}^{n_{\mathrm{con}}} \lambda_j A P^w_j A^\top \\
  p_\lambda &= A N_\lambda y
\end{align}
Find dual function by insertng $\theta^\star$ into Lagrangian:
\begin{figure*}[ht!]
  \begin{align}
    d(\lambda_{0:{n_\mathrm{con}}}) &= \mathcal{L}(\theta^\star, \lambda_{0:{n_\mathrm{con}}}) \\
    &= q^\top \Gram[\lambda]^{-1} \left( p_\lambda + q \right) - \sum_{j=0}^{n_{\mathrm{con}}} \frac{1}{2} \left( \| y_j - A_j^\top \Gram[\lambda]^{-1} \left( p_\lambda + q \right) \|_{P^\lambda_j}^2 - \Gamma_j^2 \right) \\
    &= q^\top \Gram[\lambda]^{-1} \left( p_\lambda + q \right) - \frac{1}{2} \left( \| y \|_{N_\lambda}^2 - 2 \sum_{j=0}^{n_{\mathrm{con}}} y_j^\top P_j^\lambda A_j^\top \Gram[\lambda]^{-1} (p_\lambda + q) + \| p_\lambda + q \|_{\Gram[\lambda]^{-1}}^2 - \sum_{j=0}^{n_{\mathrm{con}}} \lambda_j \Gamma_j^2 \right) \\
    &= q^\top \Gram[\lambda]^{-1} \left( p_\lambda + q \right) - \frac{1}{2} \left( \| y \|_{N_\lambda}^2 - 2 y^\top  N_\lambda A^\top \Gram[\lambda]^{-1} (p_\lambda + q) + \| p_\lambda + q \|_{\Gram[\lambda]^{-1}}^2 - \Gam[\lambda]^2 \right) \\
    &= q^\top \Gram[\lambda]^{-1} A N_\lambda y + \| q \|_{\Gram[\lambda]^{-1}}^2 + y^\top  N_\lambda A^\top \Gram[\lambda]^{-1} (A N_\lambda y + q) - \frac{1}{2} \| A N_\lambda y + q \|_{\Gram[\lambda]^{-1}}^2 - \frac{1}{2} \| y \|_{N_\lambda}^2 + \frac{1}{2} \Gam[\lambda]^2 \\
    &= \| q \|_{\Gram[\lambda]^{-1}}^2 + 2 q^\top \Gram[\lambda]^{-1} A N_\lambda y + y^\top N_\lambda A^\top \Gram[\lambda]^{-1} A N_\lambda y - \frac{1}{2} \| A N_\lambda y + q \|_{\Gram[\lambda]^{-1}}^2 - \frac{1}{2} \| y \|_{N_\lambda}^2 + \frac{1}{2} \Gam[\lambda]^2 \\
    &= \frac{1}{2} \| A N_\lambda y + q \|_{\Gram[\lambda]^{-1}}^2 - \frac{1}{2} \| y \|_{N_\lambda}^2 + \frac{1}{2} \Gam[\lambda]^2 \\
    &= \frac{1}{2} \| q \|_{\Gram[\lambda]^{-1}}^2 - \frac{1}{2} \| y \|_{N_\lambda - N_\lambda A^\top \Gram[\lambda]^{-1} A N_\lambda}^2 + \frac{1}{2} \Gam[\lambda]^2 + q^\top \Gram[\lambda]^{-1} A N_\lambda y \\
    \intertext{First, try to keep convex:}
    d(\lambda_{0:{n_\mathrm{con}}}) &= \frac{1}{2} \| q \|_{\left(\lambda_0 I_r + A N_\lambda A^\top \right)^{-1}}^2 - \frac{1}{2} \| y \|_{\left(\lambda_0^{-1} A^\top A + N_\lambda^{-1} \right)^{-1}}^2 + \frac{1}{2} \Gam[\lambda]^2 + q^\top A \lambda_0^{-1} \left(\lambda_0^{-1} A^\top A + N_\lambda^{-1} \right)^{-1} y \\
    &= \frac{1}{2} \left[ \| q \|_{\lambda_0^{-1} I_r - \lambda_0^{-1} A \left(\lambda_0^{-1} A^\top A + N_\lambda^{-1} \right)^{-1} A^\top \lambda_0^{-1}}^2 - \| y \|_{\left(\lambda_0^{-1} A^\top A + N_\lambda^{-1} \right)^{-1}}^2 + \Gam[\lambda]^2 \right] + q^\top A \lambda_0^{-1} \left(\lambda_0^{-1} A^\top A + N_\lambda^{-1} \right)^{-1} y \\
    &= \frac{1}{2} \left[ \lambda_0^{-1} \| q \|_{I_r -  A \left(A^\top A + \lambda_0 N_\lambda^{-1} \right)^{-1} A^\top}^2 - \lambda_0 \| y \|_{\left(A^\top A + \lambda_0 N_\lambda^{-1} \right)^{-1}}^2 + \Gam[\lambda]^2  \right] + q^\top A \left(A^\top A + \lambda_0 N_\lambda^{-1} \right)^{-1} y \\
    &= \frac{1}{2} \left[ \lambda_0^{-1} \| h \|_{\Sigma_{\lambda/\lambda_0}(x_{N+1})}^2 - \lambda_0 \| y \|_{\hat{K}_{\lambda/\lambda_0}^{-1}}^2 + \lambda_0 \left( \Gamma_f^2 + \sum_{j=1}^{n_{\mathrm{con}}} \frac{\lambda_j}{\lambda_0} \Gamma_{w,j}^2 \right) \right] + h^\top \mu^f_{\lambda/\lambda_0}(x_{N+1}) \\
    \intertext{At this point we can decide to trade convexity for pretty formulas and reparametrize the dual function in terms of with the following diffeomorphism:} 
    \begin{bmatrix}
      \lambda_0 \\
      \sigma_1 \\
      \vdots \\
      \sigma_{n_{\mathrm{con}}}
    \end{bmatrix}
    &= 
    \begin{bmatrix}
      \lambda_0 \\
      \sqrt{\frac{\lambda_0}{\lambda_j}} \\
      \vdots \\
      \sqrt{\frac{\lambda_0}{\lambda_{n_{\mathrm{con}}}}}
    \end{bmatrix}
    \qquad 
    \text{or vice versa:}
    \quad
    \begin{bmatrix}
      \lambda_0 \\
      \lambda_1 \\
      \vdots \\
      \lambda_{n_{\mathrm{con}}}
    \end{bmatrix}
    = 
    \begin{bmatrix}
      \lambda_0 \\
      \sigma_1^{-2} \lambda_0 \\
      \vdots \\
      \sigma_{\mathrm{con}}^{-2} \lambda_0 \\
    \end{bmatrix}
    \intertext{This allows to solve for $\lambda_0$ in the reparametrized coordinates: Since $\Gram[\lambda]^{-1} = \left( \lambda_0 I_r + A N_\lambda A^\top \right)^{-1} = \lambda_0^{-1} \left( I_r + A N_\sigma A^\top \right) = \lambda_0^{-1} M_\sigma^{-1}$, with $N_\sigma = \lambda_0^{-1} N_\lambda = \sum_{i=1}^{n_{\mathrm{con}}} \frac{\lambda_j}{\lambda_0} P_j^w$, and $\Gam[\lambda]^2 = \lambda_0 \left( \Gamma_f^2 + \sum_{i=1}^{n_{\mathrm{con}}} \frac{\lambda_j}{\lambda_0} \Gamma_{w,j} \right) = \lambda_0 \Gamma_\sigma^2$ we get} 
    d(\lambda_{0:{n_\mathrm{con}}}) &= \frac{1}{2} \left[ \lambda_0^{-1} \| q \|_{M_\sigma^{-1}}^2 + \lambda_0 \left( \Gamma_\sigma^2 - \| y \|_{(A^\top A + (N_\sigma)^{-1})^{-1}}^2 \right) \right]  + y^\top N_\sigma A^\top M_\sigma^{-1} q
    \intertext{Solve for $\nabla_{\lambda_0} d = 0$:}
    0 &= -\lambda_0^{-2} \| q \|_{M_\sigma^{-1}}^2 + \left( \Gamma_\sigma^2 - \| y \|_{(A^\top A + (N_\sigma)^{-1})^{-1}}^2 \right) \\
    \Leftrightarrow \lambda_0^2 &= \frac{\| q \|_{M_\sigma^{-1}}^2}{\Gamma_\sigma^2 - \| y \|_{(A^\top A + (N_\sigma)^{-1})^{-1}}^2}
    \intertext{Insert back into dual function:}
    d(\sigma) &= \| q \|_{M_\sigma^{-1}} \sqrt{\Gamma_\sigma^2 - \| y \|_{(A^\top A + (N_\sigma)^{-1})^{-1}}^2} + y^\top N_\sigma A^\top M_\sigma^{-1} q
  \end{align}
  
\end{figure*}

\begin{figure*}[ht!]
  Look at dual function again, convex:
  \begin{align*}
    d(\lambda_{0:n_{\mathrm{con}}}) &= \frac{1}{2} \| q \|_{\Gram[\lambda]^{-1}}^2 - \frac{1}{2} \| y \|_{N_\lambda - N_\lambda A^\top \Gram[\lambda]^{-1} A N_\lambda}^2 + \frac{1}{2} \Gam[\lambda]^2 + q^\top \Gram[\lambda]^{-1} A N_\lambda y \\
    &= \frac{1}{2} \| A N_\lambda y + q \|_{\Gram[\lambda]^{-1}}^2 - \frac{1}{2} \| y \|_{N_\lambda}^2 + \frac{1}{2} \Gam[\lambda]^2 \\
    &= \frac{1}{2} \left[ \left\| \Phi_{1:N}^\top C N_\lambda y + \Phi_{N+1}^\top h \right\|_{\Gram[\lambda]^{-1}}^2 + \left( \sum_{j=0}^{n_{\mathrm{con}}} \lambda_j \Gamma_j^2 - \| y \|_{N_\lambda}^2 \right)  \right] \\
  \end{align*}
  Insert expressions:

\end{figure*}

\subsection{Scalar kernel for multivariate measurements}

Augment the input space such that
\begin{align}
  x_k^c \doteq (x, k),
\end{align}
where $k = i$ for all training input-output pairs $(y_i, x^c_{i})$, and $k = N+1$ for an arbitrary test input location $x^c_{ N+1} \doteq (x_{N+1}, N+1)$.
Define 
\begin{align}
  k^{f,\mathrm{aug}}(x^c_{i}, x^c_{j}) \doteq c_i^\top k(x_i, x_j) c_j,
\end{align}
where $c_{N+1} \doteq h$.
Alternatively, one can also define 
\begin{align}
  k^{f,\mathrm{aug}}(x^c_{i}, x^c_{j}) \doteq C_{:,i}^\top K^f_{1:N+1,1:N+1} C_{:,j}
\end{align}
for measurement models linking different input locations,
where $C_{:,N+1} = \begin{bmatrix}
  0_{N n_f \times 1} & h^\top
\end{bmatrix}^\top$.

With a similar kernel construction, we can define a kernel for the measurement noise $w$.
Let 
\begin{align}
  k^{w,\mathrm{aug}}(x^c_i, x^c_j) \doteq e_i^\top K_\sigma^w e_j,
\end{align}
where $e_i$ is the $i$-th unit vector for training points $i \in \mathbb{I}_1^N$ and the zero vector for test points, $e_{N+1} \doteq 0$.

The kernel $k^{f,\mathrm{aug}}$ models the covariance between the $N+1$ latent scalar functions $c_i^\top f(x)$, $i = 1, \ldots, N+1$, while the kernel $k^{w,\mathrm{aug}}$ models the covariance between the scalar noise function~$w(x^c_i) \doteq w_i$.
The functions $c_i^\top f(x)$, $i = 1, \ldots, N$ are thus directly measured at their respective input locations, i.e.,
we have direct measurements 
\begin{align}
  y_i = c_i^\top f(x^c_i) + w(x^c_i), \> i=1,\ldots, N.
\end{align}

The GP posterior mean and covariance under this measurement model and correlated Gaussian measurement noise $w \sim \mathcal{N}(0, K_\sigma^w)$ is
\begin{align}
  \mu(x_{N+1}) &= K^C_{N+1, 1:N} \left( K^C_{1:N, 1:N} + K_\sigma^w \right)^{-1} y \\
  \Sigma(x_{N+1}) &= K^C_{N+1,N+1} - K^C_{N+1, 1:N} \left( K^C_{1:N, 1:N} + K_\sigma^w \right)^{-1} K^C_{1:N, N+1}
\end{align}
which is equivalent to \cref{eq:gp_definitions}.

Now, consider the RKHS norm of the minimum-norm interpolant $g \in \Hks$ of the data set in the sum-of-kernels $\ks \doteq k^{f,\mathrm{aug}} + k^{w,\mathrm{aug}}$, associated with the RKHS $\Hks$.
It holds that
\begin{align}
  \| g \|_{\Hks}^2 &= \begin{aligned}[t]
    \min_{\substack{f \in \mathcal{H}_{k^{f,\mathrm{aug}}} \\ w \in \mathcal{H}_{k^{w,\mathrm{aug}}}}} \quad & \| f \|^2_{\mathcal{H}_{k^{f,\mathrm{aug}}}} + \| w \|^2_{\mathcal{H}_{k^{w,\mathrm{aug}}}} \\
    \mathrm{s.t.} \quad & c_i^\top f(x^c_i) + w(x^c_i) = y_i
  \end{aligned} \\
  &= \begin{aligned}[t]
    \min_{\substack{c_f \in \mathbb{R}^{N} \\ c_w \in \mathbb{R}^{N}}} \quad & \| c_f \|_{(K^f_{1:N,1:N})^{-1}}^2 + \| c_w \|_{P_\sigma^w}^2 \\
    \mathrm{s.t.} \quad & C^\top c_f + c_w = y
  \end{aligned} \\
  &= \| y \|_{\hat{K}_\sigma^{-1}}^2,
\end{align}
where the first equality follows from the definition of the RKHS norm for the sum of kernels~\cite[Chapter~6]{aronszajn_theory_1950}, the second equality from the representer theorem~\cite{kimeldorf_results_1971,goos_generalized_2001} and the third equality from the solution of the convex quadratic program.

\subsection{Alternative invexity proof}

Recap invexity:
\begin{definition}[Invexity, cf.~\cite{mishra_invexity_2008}, Sec.~2.2]
  Let $\Omega \subseteq \mathbb{R}^{n}$ be an open set. 
  The differentiable function $f: \Omega \rightarrow \mathbb{R}^{}$ is invex if there exists a vector function $\eta: \Omega \times \Omega \rightarrow \mathbb{R}^{n}$ such that
  \begin{align}
    f(x) - f(y) \geq \eta(x,y)^\top \nabla f(y), \quad \forall x,y \in \Omega.
  \end{align}
\end{definition}

\begin{lemma}[\cite{craven1981duality}, cf.~\cite{mishra_invexity_2008}, p.~12]
  \label{thm:invexity_invertible_jacobian}
  Let $g: \mathbb{R}^{m} \rightarrow \mathbb{R}^{}$ be differentiable and convex and $h: \mathbb{R}^{n} \rightarrow \mathbb{R}^{m}$, $n \geq m$, be differentiable and the $\nabla h$ be of rank~$m$. Then, $f = g \circ h$ is invex.
\end{lemma}

\begin{proof}
  Due to convexity of $g$ all $x, y \in \mathbb{R}^{n}$, it holds that
  \begin{align}
    f(x) - f(y) &= g(h(x)) - g(h(y)) \\
    &\geq (h(x) - h(y))^\top \nabla g(h(y)).
  \end{align}
  Consider now the equation
  \begin{align}
    (h(x) - h(y))^\top \nabla g(h(y)) &= \eta(x,y)^\top \nabla f(y), \\
    &= \eta(x,y)^\top \nabla h(y) \nabla g(h(y)),
  \end{align}
  which is satisfied if there exists $\eta(x,y)^\top \in \mathbb{R}^{n}$ satisfying
  \begin{align}
    \eta(x,y)^\top \nabla h(y) &= (h(x) - h(y))^\top.
  \end{align}
  This is guaranteed since $\nabla h(y) \in \mathbb{R}^{n \times m}$ is of rank~$m$.
\end{proof}

Dual function $d(\lambda_{0:N})$ is convex.
Minimizing the single coordinate $\lambda_0$ over a convex set $\lambda_0 \geq 0$ preserves convexity~\cite[Section~3.2.5]{boyd_convex_2004}, i.e., $d_0(\lambda_{1:N}) = d(\lambda_0^\star(\lambda_{1:N}), \lambda_{1:N})$ is convex.
Introduce diffeomorphism
\begin{align}
  \Phi(\sigma)
  &=
  \begin{bmatrix}
    \lambda_1 \\
    \vdots \\
    \lambda_N
  \end{bmatrix}
  =
  \lambda_0^\star(\sigma_{1:N})
  \begin{bmatrix}
    \frac{1}
    {\sigma_1^2} \\
    \vdots \\
    \frac{1}{\sigma_N^2}
  \end{bmatrix},
  \intertext{where $\lambda_0^\star(\sigma_{1:N})$ is the optimal dual parameter of the relaxed problem~\eqref{eq:inf_relax_findim}. The inverse is given by}
  \Phi^{-1}(\lambda_{1:N}) 
  &= 
  \begin{bmatrix}
    \sigma_1 \\
    \vdots \\
    \sigma_N
  \end{bmatrix}
  =
  \sqrt{\lambda_0^\star(\lambda_{1:N})}
  \begin{bmatrix}
    \sqrt{\frac{1}{\lambda_1}} \\
    \vdots \\
    \sqrt{\frac{1}{\lambda_N}}
  \end{bmatrix},
  \intertext{where $\lambda_0^\star(\lambda_{1:N}) = \lambda_0^\star(\sigma_{1:N})$ is obtained by minimizing $d(\lambda_{0:N})$ along its first coordinate. Since $\lambda_0^\star(\sigma_{1:N}) > 0$ the Jacobian}
  \frac{\partial \Phi(\sigma)}{\partial \sigma}
  &= 
  \begin{aligned}[t]
    & - 2 \lambda_0^\star(\sigma_{1:N})
    \begin{bmatrix}
      \frac{1}{\sigma_1^{3}}  & & \\
      & \ddots & \\
      & & \frac{1}{\sigma_N^{3}}
    \end{bmatrix} \\
    & + 
    \begin{bmatrix}
      \frac{1}
      {\sigma_1^2} \\
      \vdots \\
      \frac{1}{\sigma_N^2}
    \end{bmatrix}
    \begin{bmatrix}
      \frac{\partial \lambda_0(\sigma_{1:N})}{\partial \sigma_1} & \ldots & \frac{\partial \lambda_0(\sigma_{1:N})}{\partial \sigma_N}
    \end{bmatrix}
  \end{aligned} \\
  &= \Sigma_{\partial} + u v^\top
\end{align}
is invertible(?). The inverse is given as
\begin{align}
  \left( \frac{\partial \Phi(\sigma)}{\partial \sigma} \right)^{-1} = \Sigma_{\partial}^{-1} - \Sigma_{\partial}^{-1} u \left( 1 + v^\top u \right)^{-1} v^\top \Sigma_{\partial}^{-1} 
\end{align}
and exists as long as $1 + v^\top u \neq 0$.
Thus, we have that
\begin{align}
  d_0^\sigma(\sigma_{1:N}) &= d_0(\Phi(\sigma_{1:N})),
\end{align}
where $d_0$ is convex and $\Phi$ has an invertible Jacobian,
implying
invexity of $d_0^\sigma$ by \cref{thm:invexity_invertible_jacobian}.
The function $\eta(\sigma_1, \sigma_2)$ is given by
\begin{align}
  \eta(\sigma_1, \sigma_2) &= \left( \frac{\partial \Phi(\sigma_1)}{\partial \sigma} \right)^{-1} \left( \Phi(\sigma_2) - \Phi(\sigma_1) \right).
\end{align}

\subsection{Scharnhorst dual formulation}

\begin{align}
  \begin{aligned}
    \overline{f}_h(x_{N+1}) = \min_{\substack{\nu \in \mathbb{R}^{N}, \\ \lambda > 0}} \quad & \frac{1}{4 \lambda} \bigg( K^f_{N+1,N+1} + \| C \nu \|_{K^f_{1:N,1:N}}^2 \\
    & - 2 K^f_{N+1,1:N} C \nu \bigg) \\
    & + y^\top \nu + \bar{\Gamma}_{w} \| \nu \|_1 + \lambda \Gamma_f^2
  \end{aligned}
\end{align}
Using that $K^f_{1:N+1,1:N+1} = \Phi_{1:N+1} \Phi_{1:N+1}^\top$, we obtain
\begin{align}
  & K^f_{N+1,N+1} + \| C \nu \|_{K^f_{1:N,1:N}}^2 - 2 K^f_{N+1,1:N} C \nu \\
  =& \> \Phi_{N+1} \Phi_{N+1}^\top + \| \Phi_{1:N}^\top C \nu \|_2^2 - 2 \Phi_{N+1} \Phi_{1:N}^\top C \nu \\
  =& \> \| \Phi_{N+1}^\top - \Phi_{1:N}^\top C \nu \|_2^2
\end{align}
Solve for $\lambda$:
\begin{align}
  & \min_{\lambda > 0} \quad \frac{\| \Phi_{N+1}^\top - \Phi_{1:N}^\top C \nu \|_2^2}{4 \lambda} + \lambda \Gamma_f^2 \\
  \Rightarrow 0 &= \Gamma_f^2 - \frac{\| \Phi_{N+1}^\top - \Phi_{1:N}^\top C \nu \|_2^2}{4 \lambda^2} \\
  \Rightarrow \lambda^\star &= \frac{\| \Phi_{N+1}^\top - \Phi_{1:N}^\top C \nu \|_2}{2 \Gamma_f}
\end{align}
Insert into cost:
\begin{align}
  & \frac{\| \Phi_{N+1}^\top - \Phi_{1:N}^\top C \nu \|_2^2}{4 \lambda^\star} + \lambda^\star \Gamma_f^2 \\
  = & \frac{1}{2} \Gamma_f \| \Phi_{N+1}^\top - \Phi_{1:N}^\top C \nu \|_2 + \frac{1}{2} \Gamma_f \| \Phi_{N+1}^\top - \Phi_{1:N}^\top C \nu \|_2 \\
  = & \Gamma_f \| \Phi_{N+1}^\top - \Phi_{1:N}^\top C \nu \|_2
\end{align}
Original problem:
\begin{align}
  \overline{f}_h(x_{N+1}) &=
  \begin{aligned}[t]  
    \min_{\substack{\nu \in \mathbb{R}^{N}}} \quad & \frac{1}{4 \lambda^\star(\nu)} \| \Phi_{N+1}^\top - \Phi_{1:N}^\top C \nu \|_2^2 \\
    & + y^\top \nu + \bar{\Gamma}_{w} \| \nu \|_1 + \lambda^\star(\nu) \Gamma_f^2
  \end{aligned} \\
  &= \begin{aligned}[t]  
    \min_{\substack{\nu \in \mathbb{R}^{N}}} \quad & \Gamma_f \| \Phi_{N+1}^\top - \Phi_{1:N}^\top C \nu \|_2 + y^\top \nu + \bar{\Gamma}_{w} \| \nu \|_1
  \end{aligned} \\
  &= \begin{aligned}[t]  
    \min_{\substack{\nu \in \mathbb{R}^{N}}} \quad & \Gamma_f \left\| \begin{bmatrix} 
    - C \nu \\
    1
    \end{bmatrix} \right\|_{K^f_{1:N+1,1:N+1}} + y^\top \nu + \bar{\Gamma}_{w} \| \nu \|_1
  \end{aligned}
\end{align}

\subsection{Scharnhorst dual multivariate derivation}

\begin{subequations}
  \label{eq:inf_opt_findim_scharnhorst}
  \begin{align}
    \underline{f}_h(x_{N+1}) = \inf_{\substack{\theta \in\mathbb{R}^{r}}} \quad & h^\top \Phi_{N+1} \theta                                                                                    \\
    \mathrm{s.t.} \quad  
    & \| \theta \|_2^2 \leq \Gamma_f^2,                                                                        \\
    & y_i - c_i^\top \Phi_{1:N} \theta \leq \Gamma_{w,i}, \\
    & -y_i + c_i^\top \Phi_{1:N} \theta \leq \Gamma_{w,i},
  \end{align}
\end{subequations}
Abbreviations:
\begin{subequations}
  \label{eq:inf_opt_findim_scharnhorst_short}
  \begin{align}
    \underline{f}_h(x_{N+1}) = \inf_{\substack{\theta \in\mathbb{R}^{r}}} \quad & q^\top \theta                                                                                    \\
    \mathrm{s.t.} \quad  
    & \| \theta \|_2^2 \leq \Gamma_f^2,                                                                        \\
    & y - A^\top \theta \leq \bar{\Gamma}_{w}, \\
    & -y + A^\top \theta \leq \bar{\Gamma}_{w},
  \end{align}
\end{subequations}
Lagrangian:
\begin{align*}
  \mathcal{L}(\theta, \lambda, \beta, \gamma) &= 
  \begin{aligned}[t]
    & q^\top \theta + \lambda \left( \| \theta \|_2^2 - \Gamma_f^2 \right) \\
    & + \beta^\top \left( y - A^\top \theta - \bar{\Gamma}_{w} \right) \\
    & + \gamma^\top \left( -y + A^\top \theta - \bar{\Gamma}_{w} \right)
  \end{aligned} \\
  &= 
  \begin{aligned}[t]
    & q^\top \theta + \lambda \left( \| \theta \|_2^2 - \Gamma_f^2 \right) \\
    & + (\beta^\top - \gamma^\top) \left( y - A^\top \theta \right) \\
    & - (\beta^\top + \gamma^\top) \bar{\Gamma}_{w}
  \end{aligned}
\end{align*}
Stationarity, $\nabla_\theta \mathcal{L} = 0$:
\begin{align*}
  0 &= q + 2 \lambda \theta - A (\beta - \gamma) \\
  \Leftrightarrow \theta^\star &= \frac{1}{2 \lambda} \left( A (\beta - \gamma) - q \right)
\end{align*}
Dual function:
\begin{align*}
  d(\lambda, \beta, \gamma) &= 
  \begin{aligned}[t]
    & \frac{1}{4 \lambda} \left( 
      \| q \|_2^2 - 2 q^\top A (\beta - \gamma) + \| A (\beta - \gamma) \|_2^2
    \right) \\
    & + \frac{1}{2 \lambda} \left( 
      - \| q \|_2^2 
      + q^\top A (\beta - \gamma) 
    \right) \\
    & + \frac{1}{2 \lambda} \left( 
      - \| A (\beta - \gamma) \|_2^2
      + q^\top A (\beta - \gamma)
    \right) \\
    & + y^\top (\beta - \gamma) - \lambda \Gamma_f^2 - \bar{\Gamma}_w^\top (\beta + \gamma)
  \end{aligned} \\
  &= 
  \begin{aligned}[t]
    & \frac{1}{4 \lambda} \left( 
      \| q \|_2^2 - 2 q^\top A (\beta - \gamma) + \| A (\beta - \gamma) \|_2^2
    \right) \\
    & - \frac{1}{2 \lambda} \left( 
      \| q \|_2^2 
      - 2 q^\top A (\beta - \gamma)
      + \| A (\beta - \gamma) \|_2^2
    \right) \\
    & + y^\top (\beta - \gamma) - \lambda \Gamma_f^2 - \bar{\Gamma}_w^\top (\beta + \gamma)
  \end{aligned} \\
  &= 
  \begin{aligned}[t]
    & -\frac{1}{4 \lambda} \left( 
      \| q - A (\beta - \gamma) \|_2^2
    \right) \\
    & + y^\top (\beta - \gamma) - \lambda \Gamma_f^2 - \bar{\Gamma}_w^\top (\beta + \gamma)
  \end{aligned}
\end{align*}
Reparametrize with invertible linear transformation:
\begin{align*}
  \begin{bmatrix}
    \nu \\
    \gamma
  \end{bmatrix}
  &= 
  \begin{bmatrix}
    I & -I \\
    0 & I
  \end{bmatrix}
  \begin{bmatrix}
    \beta \\
    \gamma
  \end{bmatrix}
  &&
  \Leftrightarrow
  &&
  \begin{bmatrix}
    \beta \\
    \gamma
  \end{bmatrix}
  &= 
  \begin{bmatrix}
    I & I \\
    0 & I
  \end{bmatrix}
  \begin{bmatrix}
    \nu \\
    \gamma
  \end{bmatrix}
\end{align*}
Transform constraints:
\begin{align*}
  \begin{bmatrix}
    \beta \\
    \gamma
  \end{bmatrix}
  \geq 0 
  &&
  \Leftrightarrow
  && 
  \begin{bmatrix}
    \nu + \gamma \\
    \gamma 
  \end{bmatrix}
  \geq 0
\end{align*}
Insert into dual function:
\begin{align*}
  d(\lambda, \beta, \gamma) &= 
  \begin{aligned}[t]
    & -\frac{1}{4 \lambda} \left( 
      \| q - A \nu \|_2^2
    \right) \\
    & + y^\top \nu - \lambda \Gamma_f^2 - \bar{\Gamma}_w^\top (\nu + 2 \gamma) 
  \end{aligned}
\end{align*}
Maximize dual function:
\begin{align*}
  \sup_{\substack{
    \nu \in \mathbb{R}^{N} \\
    \gamma \in \mathbb{R}^{N} \geq 0 \\
    \lambda > 0
  }}
  & \quad
  -\frac{1}{4 \lambda} \left( 
    \| q - A \nu \|_2^2
  \right)
  + y^\top \nu - \lambda \Gamma_f^2 - \bar{\Gamma}_w^\top (\nu + 2 \gamma) \\
  \mathrm{s.t.} \quad & \nu + \gamma \geq 0
\end{align*}
Partial maximization over $\lambda$ leads to:
\begin{align*}
  \sup_{\substack{
    \nu \in \mathbb{R}^{N} \\
    \gamma \in \mathbb{R}^{N}
  }}
  \quad &
  - \Gamma_f \| q - A \nu \|_2
  + y^\top \nu - \bar{\Gamma}_w^\top (\nu + 2 \gamma) \\
  \mathrm{s.t.} \quad & \nu + \gamma \geq 0 \\
  & \gamma \geq 0
\end{align*} 
We first solve partially for $\gamma$ as a function of $\nu$. 
The constraints imply that $\gamma \geq -\nu$ and that $\gamma \geq 0$.
Solving 
\begin{align*}
  \sup_{\substack{
    \gamma \in \mathbb{R}^{N}
  }}
  \quad & - 2 \bar{\Gamma}_w^\top \gamma \\
  \mathrm{s.t.} \quad & \gamma \geq - \nu \\
  & \gamma \geq 0
\end{align*}
thus leads to $\gamma^\star(\nu) = \max(-\nu, 0)$.
Inserting $\gamma^\star$ into the objective yields 
$\nu + 2 \max(-\nu, 0) = \| \nu \|_1$;
inserting $\gamma^\star$ into the constraints,
$\nu \geq -\max(-\nu, 0) = \min(\nu, 0)$, which is trivially fulfilled.
Finally, we obtain the dual problem
\begin{align*}
  \sup_{\substack{
    \nu \in \mathbb{R}^{N}  
  }}
  \quad &
  - \Gamma_f \| q - A \nu \|_2
  + y^\top \nu - \bar{\Gamma}_w \| \nu \|_1
\end{align*} 
Inserting the abbreviated expressions yields:
\begin{align*}
  \sup_{\substack{
    \nu \in \mathbb{R}^{N}  
  }}
  \quad &
  - \Gamma_f \| \Phi_{N+1}^\top h - \Phi_{1:N}^\top C \nu \|_2
  + y^\top \nu - \bar{\Gamma}_w \| \nu \|_1
\end{align*} 
As above, this is the same as
\begin{align*}
  \sup_{\substack{
    \nu \in \mathbb{R}^{N}  
  }}
  \quad &
  - \Gamma_f \| \begin{bmatrix} -C \nu \\ h \end{bmatrix} \|_{K^f_{1:N+1,1:N+1}}
  + y^\top \nu - \bar{\Gamma}_w \| \nu \|_1
\end{align*} 

\subsection{Scharnhorst iterative implementation}

unconstrained quadrativ program:
\begin{align*}
  \min_{\nu \in \mathbb{R}^{N}} \quad & \frac{1}{4 \lambda^\star} \| C \nu \|_{K^f_{1:N,1:N}}^2 + \left( y - \frac{1}{2 \lambda^\star} C^\top  \right)
\end{align*}

\subsection{Closed-form bound comparison}

Maddalena Prop.~1:
\begin{align*}
  \| \fmu \|_{\Hkf}^2
\end{align*}

Scharnhorst Prop.~3 \cite{scharnhorst_robust_2023}:
\begin{align*}
  \begin{aligned}[t]
    aa
  \end{aligned}
\end{align*}

Maddalena~\cite{maddalena_deterministic_2021}:
\begin{align*}
  \begin{aligned}[t]
    | \ftr(x_{N+1}) - \fmu[\sigma^2 / N](x_{N+1}) | \leq & \> \Sigma^f_0(x_{N+1}) \sqrt{\Gamma_f^2 - \| \gmu[0] \|_{\Hkf}^2} \\
    & \> + [\Gamma_{w,j}]^\top | (K^f_{1:N,1:N})^{-1} K^f_{N+1,1:N} | \\
    & \> + \left| y^\top \left( K^f_{1:N,1:N} + \frac{(K^f_{1:N,1:N})^2}{\sigma^2}  \right)^{-1} K^f_{N+1,1:N} \right|
  \end{aligned}
\end{align*}

\fi

\end{document}